\documentclass[letterpaper, 10 pt, conference]{ieeeconf}  

\IEEEoverridecommandlockouts                              
\usepackage{amsmath,amsfonts,amssymb,url,listings,mathrsfs,wasysym}
\usepackage{enumerate}
\usepackage{graphicx}
\usepackage{subfigure}
\usepackage{paralist}
\usepackage[version=0.96]{pgf}
\usepackage{tikz}
\usetikzlibrary{arrows,automata}

\newcommand{\num}[1]{\relax\ifmmode \mathbb #1\else $\mathbb #1$\fi}

\newcommand{\oor}{\vee}
\newcommand{\aand}{\wedge}
\newcommand{\until}{\hspace{1mm}\mathcal{U}\hspace{1mm}}
\newcommand{\always}{\square}
\newcommand{\eventually}{\Diamond}
\renewcommand{\next}{\ocircle}
\newcommand{\true}{\mathit{True}}

\renewcommand{\i}{\iota}
\newcommand{\A}{{\mathcal{A}}}
\newcommand{\C}{{\mathcal{C}}}

\newcommand{\lang}{\mathcal{L}}
\newcommand{\M}{{\mathcal{M}}}

\newcommand{\T}{{\mathcal{T}}}

\newcommand{\prob}{\mathrm{Pr}}

\newtheorem{lemma}{Lemma}

\newtheorem{corollary}{Corollary}
\newtheorem{definition}{Definition}

\newtheorem{example}{Example}

\title{\LARGE \bf
Incremental Temporal Logic Synthesis of Control Policies 
for Robots Interacting with Dynamic Agents
}
\author{Tichakorn Wongpiromsarn, Alphan Ulusoy, Calin Belta, Emilio Frazzoli and Daniela Rus\vspace{-5mm}
\thanks{T. Wongpiromsarn is with the Singapore-MIT Alliance for Research and Technology,
Singapore 117543, Singapore. {\tt\footnotesize nok@smart.mit.edu}}
\thanks{A. Ulusoy and C. Belta are with Boston University,
Boston, MA, USA {\tt\footnotesize alphan@bu.edu, cbelta@bu.edu}}%
\thanks{E. Frazzoli and D. Rus are with the Massachusetts Institute of Technology,
Cambridge, MA, USA {\tt\footnotesize frazzoli@mit.edu, rus@csail.mit.edu}}%
}

\begin{document}
\maketitle
\thispagestyle{empty}
\pagestyle{empty}

\begin{abstract}
We consider the synthesis of control policies from temporal logic specifications
for robots that interact with multiple dynamic environment agents.
Each environment agent is modeled by a Markov chain whereas
the robot is modeled by a finite transition system (in the deterministic case) or Markov decision process (in the stochastic case).
Existing results in probabilistic verification are adapted to solve the synthesis problem.
To partially address the state explosion issue,
we propose an incremental approach where only a small subset of environment agents is incorporated in the synthesis procedure initially
and more agents are successively added until we hit the constraints on computational resources.
Our algorithm runs in an anytime fashion where the probability that the robot satisfies its specification increases
as the algorithm progresses.
\end{abstract}

\section{Introduction}
Temporal logics \cite{Allen:TML90,zohar92,Baier:PMC2008} have been recently employed to precisely express complex behaviors of robots.
In particular, given a robot specification expressed as a formula in a temporal logic,
control policies that ensure or maximize the probability that the robot
satisfies the specification can be automatically synthesized 
based on exhaustive exploration of the state space
\cite{fainekos05,kress-gazit07,belta2007symbolic,conner07rsj,karaman09cdc,Bhatia:SMP2010,Ding:LTL2011,Medina11rsj,KWT-RAM}.
Consequently, the main limitation of existing approaches for synthesizing control policies
from temporal logic specifications is almost invariably due to a combinatorial blow up of the state space, 
commonly known as the state explosion problem.

In many applications, robots need to interact with external, potentially dynamic agents, including human and other robots.
As a result, the control policy synthesis problem becomes more computationally complex as more external agents
are incorporated in the synthesis procedure.
Consider, as an example, the problem where an autonomous vehicle needs to go through a pedestrian crossing while there
are multiple pedestrians who are already at or approaching the crossing.
The state space of the complete system (i.e., the vehicle and all the pedestrians)
grows exponentially with the number of the pedestrians.
Hence, given a limited budget of computational resources,
solving the control policy synthesis problem with respect to temporal logic specifications
may not be feasible when there are a large number of pedestrians.

In this paper, we partially address the aforementioned issue and
propose an algorithm for computing a robot control policy in an anytime manner.
Our algorithm progressively computes a sequence of control policies, taking into account only a small subset of the environment agents initially 
and successively adds more agents to the synthesis procedure in each iteration until the computational resource constraints are exceeded.
As opposed to existing incremental synthesis approaches that handle temporal logic specifications where
representative robot states are incrementally added to the synthesis procedure \cite{karaman09cdc},
we consider incrementally adding representative environment agents instead.

The main contribution of this paper is twofold.
First, we propose an anytime algorithm for synthesizing a control policy for a robot interacting with multiple environment agents
with the objective of maximizing the probability for the robot to satisfy a given temporal logic specification.
Second, an incremental construction of various objects needed to be computed during the synthesis procedure
is proposed.
Such an incremental construction makes our anytime algorithm more efficient
by avoiding unnecessary computation and exploiting the objects computed in the previous iteration.
Experimental results show that not only we obtain a reasonable solution much faster, 
but we are also able to obtain an optimal solution faster than existing approaches.

The rest of the paper is organized as follows: 
We provide useful definitions and descriptions of the formalisms in the following section. 
Section \ref{sec:prob} is dedicated to the problem formulation. 
Section \ref{sec:syn} provides a complete solution to the control policy synthesis problem
for robots that interact with environment agents.
Incremental computation of control policies is discussed in Section \ref{sec:incremental}.
Section \ref{sec:example} presents experimental results.
Finally, Section \ref{sec:conclusions} concludes the paper and discusses future work.

\section{Preliminaries}
We consider systems that comprise multiple (possibly stochastic) components.
In this section, we define the formalisms used in this paper to describe such systems and their desired properties.
Throughout the paper, we let $X^*$, $X^\omega$ and $X^+$ denote the set of 
finite, infinite and nonempty finite strings, respectively, of a set $X$.

\subsection{Automata}
\begin{definition} 
A \emph{deterministic finite automaton}  (DFA) is a tuple
$\A = (Q, \Sigma, \delta, q_{init}, F)$ where
\begin{itemize}
\item $Q$ is a finite set of states,
\item $\Sigma$ is a finite set called alphabet,
\item $\delta: Q \times \Sigma \to Q$ is a transition function,
\item $q_{init} \in Q$ is the initial state, and
\item $F \subseteq Q$ is a set of final states.
\end{itemize}
We use the relation notation, $q \stackrel{w}{\longrightarrow} q'$
to denote $\delta(q,w) = q'$.
\end{definition}

Consider a finite string $\sigma = \sigma_1\sigma_2\ldots \sigma_n \in \Sigma^*$.
A \emph{run} for $\sigma$ in a DFA $\A = (Q, \Sigma, \delta, q_{init}, F)$ is a finite sequence of states
$q_0q_1\ldots q_n$ such that $q_0 = q_{init}$ and 
$q_0 \stackrel{\sigma_1}{\longrightarrow} q_1 \stackrel{\sigma_2}{\longrightarrow} q_2 
\stackrel{\sigma_3}{\longrightarrow} \ldots \stackrel{\sigma_n}{\longrightarrow} q_n$.
A run is \emph{accepting} if $q_n \in F$.
A string $\sigma \in \Sigma^*$ is \emph{accepted} by $\A$ if there is an accepting run of $\sigma$ in $\A$.
The language \emph{accepted} by $\A$, denoted by $\lang(\A)$, is the set of all accepted strings of $\A$.

\subsection{Linear Temporal Logic}
Linear temporal logic (LTL) is a branch of logic that can be used to reason about a time line.
An LTL formula is built up from a set $\Pi$ of atomic propositions,
the logic connectives $\neg$, $\oor$, $\aand$ and $\Longrightarrow$
and the temporal modal operators 
$\next$ (``next''), $\always$ (``always''), $\eventually$ (``eventually'') and $\until$ (``until'').
An LTL formula over a set $\Pi$ of atomic propositions is inductively defined as
\[
  \varphi := \true \hspace{1mm}|\hspace{1mm} p \hspace{1mm}|\hspace{1mm} \neg \varphi \hspace{1mm}|\hspace{1mm}
   \varphi \aand \varphi \hspace{1mm}|\hspace{1mm} \next \varphi \hspace{1mm}|\hspace{1mm} \varphi \until \varphi 
\]
where $p \in \Pi$.
Other operators can be defined as follows: 
$\varphi \aand \psi = \neg(\neg\varphi \oor \neg\psi)$, 
$\varphi \Longrightarrow \psi = \neg \varphi \oor \psi$,
$\eventually\varphi = \true \until \varphi$, and
$\always\varphi = \neg\eventually\neg\varphi$.

\textit{\textbf{Semantics of LTL:}}
LTL formulas are interpreted on infinite strings over $2^{\Pi}$.
Let $\sigma = \sigma_0 \sigma_1 \sigma_2 \ldots$ where $\sigma_i \in 2^{\Pi}$ for all $i \geq 0$.
The	satisfaction relation $\models$	is defined inductively on LTL formulas as follows:
\begin{itemize}
\item $\sigma \models \true$,
\item for an atomic proposition $p \in \Pi$, $\sigma \models p$ if and only if $p \in \sigma_0$,
\item $\sigma \models \neg \varphi$ if and only if $\sigma \not\models \varphi$,
\item $\sigma \models \varphi_1 \aand \varphi_2$ if and only if $\sigma \models \varphi_1$ and $\sigma \models \varphi_2$,
\item $\sigma \models \next\varphi$ if and only if $\sigma_1 \sigma_2\ldots \models \varphi$, and
\item $\sigma \models \varphi_1 \until \varphi_2$ if and only if there exists $j \geq 0$ such that 
$\sigma_j \sigma_{j+1} \ldots \models \varphi_2$ and for all $i$ such all $0 \leq i < j$, $\sigma_i \sigma_{i+1}\ldots \models \varphi_1$.
\end{itemize}

More details on LTL can be found, e.g., in \cite{Allen:TML90,zohar92,Baier:PMC2008}.

In this paper, we are particularly interested in a class of LTL known as
co-safety formulas.
An important property of a co-safety formula is that any word
satisfying the formula has a finite good prefix, i.e., a finite prefix that cannot be extended to violate the formula.
Specifically, given an alphabet $\Sigma$, 
a language $L \subseteq \Sigma^\omega$ is co-safety if and only if
every $w \in L$ has a good prefix $x \in \Sigma^*$ such that for all
$y \in \Sigma^\omega$, we have $x \cdot y \in L$.
In general, the problem of determining whether an LTL formula is co-safety is PSPACE-complete \cite{Kupferman:MSC2001}.
However, there is a class of co-safety formulas, known as \emph{syntactically co-safe} LTL formulas,
which can be easily characterized.
A \emph{syntactically co-safe} LTL formula over $\Pi$
is an LTL formula over $\Pi$ whose only temporal operators are $\next$, $\eventually$ and $\until$	
when written in	positive normal	form where the negation operator $\neg$ occurs 
only in front of atomic propositions \cite{Baier:PMC2008,Kupferman:MSC2001}.
It can be shown that for any syntactically co-safe formula $\varphi$,
there exists a DFA $\A_\varphi$ that accepts all and only words in $pref(\varphi)$, i.e.,
$\lang(\A_\varphi) = pref(\varphi)$,
where $pref(\varphi)$ denote the set of all good prefixes for $\varphi$ \cite{Bhatia:SMP2010}.

\subsection{Systems and Control Policies}
We consider the case where each component of the system can be modeled by a deterministic finite transition system,
Markov chain or Markov decision process, depending on the characteristics of that component.
These different models are defined as follows.

\begin{definition}
	A \emph{deterministic finite transition system (DFTS)} is a tuple
	$\T = (S, Act, \longrightarrow, s_{init}, \Pi, L)$ where
	\begin{itemize}
	\item $S$ is a finite set of states,
	\item $Act$ is a finite set of actions,
	\item $\longrightarrow \subseteq S \times Act \times S$ is a transition relation such that
	for all $s \in S$ and $\alpha \in Act$, $|Post(s, \alpha)| \leq 1$ where
	$Post(s, \alpha) = \{s' \in S \hspace{1mm}|\hspace{1mm} (s, \alpha, s') \in \longrightarrow\}$,
	\item $s_{init} \in S$ is the initial state,
	\item $\Pi$ is a set of atomic propositions, and
	\item $L : S \to 2^\Pi$ is a labeling function.
	\end{itemize}
	
	$(s, \alpha, s') \in \longrightarrow$ is denoted by $s \stackrel{\alpha}{\longrightarrow} s'$.
	An action $\alpha$ is \emph{enabled} in state $s$ if and only if there exists $s'$ such that
	$s \stackrel{\alpha}{\longrightarrow} s'$.
\end{definition}

\begin{definition}
	A \emph{(discrete-time) Markov chain} (MC) is a tuple $\M = (S, \mathbf{P}, \i_{init}, \Pi, L)$ where
	$S$, $\Pi$ and $L$ are defined as in DFTS and
	\begin{itemize}
	\item $\mathbf{P} : S \times S \to [0,1]$ is the transition probability function such that for any state $s \in S$,
	$\sum_{s' \in S} \mathbf{P}(s, s') = 1$, and
	\item $\i_{init} : S \to [0,1]$ is the initial state distribution satisfying $\sum_{s \in S} \i_{init}(s) = 1$.
	\end{itemize}
\end{definition}

\begin{definition}
	A \emph{Markov decision process} (MDP) is a tuple
	$\M = (S, Act, \mathbf{P}, \i_{init}, \Pi, L)$ where
	$S$, $Act$, $\i_{init}$, $\Pi$ and $L$ are defined as in DFTS and MC and
	$\mathbf{P} : S \times Act \times S \to [0,1]$ is the transition probability function such that for any state $s \in S$ and action $\alpha \in Act$,
	$\sum_{s' \in S} \mathbf{P}(s, \alpha, s') \in \{0, 1\}$.

	An action $\alpha$ is \emph{enabled} in state $s$ if and only if $\sum_{s' \in S} \mathbf{P}(s, \alpha, s') = 1$.
	Let $Act(s)$ denote the set of enabled actions in $s$.
\end{definition}

Given a complete system as the composition of all its components, we are interested in computing
a control policy for the system that optimizes certain objectives.
We define a control policy for a system modeled by an MDP as follows.

\begin{definition}
	Let $\M = (S, Act, \mathbf{P}, \i_{init}, \Pi, L)$ be a Markov decision process.
	A \emph{control policy} for $\M$ is a function $\C : S^+ \to Act$
	such that $\C(s_0s_1\ldots s_n) \in Act(s_n)$ for all $s_0s_1\ldots s_n \in S^+$.
\end{definition}

Let $\M = (S, Act, \mathbf{P}, \i_{init}, \Pi, L)$ be an MDP
and $\C : S^+ \to Act$ be a control policy for $\M$.
Given an initial state $s_0$ of $\M$ such that $\i_{init}(s_0) > 0$, 
an infinite sequence $r_\M^\C = s_0 s_1 \ldots$ on $\M$ generated under policy $\C$
is called a \emph{path} on $\M$ if $\mathbf{P}(s_i, \C(s_0s_1\ldots s_i), s_{i+1}) > 0$ for all $i$.
The subsequence $s_0s_1 \ldots s_n$ where $n \geq 0$ is
the \emph{prefix} of length $n$ of $r_\M^\C$.
We define $Paths_\M^\C$ and $FPaths_\M^\C$ as the set of all infinite paths of $\M$
under policy $\C$ and their finite prefixes, respectively, 
starting from any state $s_0$ with $\i_{init}(s_0) > 0$.
For $s_0s_1\ldots s_n \in FPaths_\M^\C$, we let $Paths_\M^\C(s_0s_1\ldots s_n)$
denote the set of all paths in $Paths_\M^\C$ with the prefix $s_0s_1 \ldots s_n$.

The $\sigma$-algebra associated with $\M$ under policy $\C$
is defined as the smallest $\sigma$-algebra that contains $Paths_\M^\C(\hat{r}_\M^\C)$
where $\hat{r}_\M^\C$ ranges over all finite paths in $FPaths_\M^\C$.
It follows that there exists a unique probability measure $Pr_\M^\C$ on the $\sigma-$algebra 
associated with $\M$ under policy $\C$ where for any $s_0s_1\ldots s_n \in FPaths_\M^\C$,
\begin{equation*}
  \begin{array}{l}
  \prob_\M^\C\{Paths_\M^\C(s_0s_1\ldots s_n)\} = \\
  \hspace{10mm}\i_{init}(s_0) \prod_{0 \leq i < n} \mathbf{P}(s_i, \C(s_0s_1\ldots s_i), s_{i+1}).
  \end{array}
\end{equation*}

Given an LTL formula $\varphi$, one can show that the set 
$\{s_0 s_1 \ldots \in Paths_\M^\C \hspace{1mm}|\hspace{1mm} L(s_0) L(s_1) \ldots \models \varphi\}$
is measurable \cite{Baier:PMC2008}.
The probability for $\M$ to satisfy $\varphi$ under policy $\C$ is then defined as
\begin{equation*}
  \prob_\M^\C(\varphi) = \prob_\M^\C\{s_0 s_1 \ldots \in Paths_\M^\C \hspace{1mm}|\hspace{1mm}
  L(s_0) L(s_1) \ldots \models \varphi\} .
\end{equation*}

For a given (possibly noninitial) state $s \in S$, we let $\M^s = (S, Act, \mathbf{P}, \i^s_{init}, \Pi, L)$ where 
$\i^s_{init}(t) = 1$ if $s = t$ and $\i^s_{init}(t) = 0$ otherwise.
We define $\prob_\M^\C(s \models \varphi) = \prob_{\M^s}^\C(\varphi)$ 
as the probability for $\M$ to satisfy $\varphi$ under policy $\C$, starting from $s$.

A control policy essentially resolves all nondeterministic choices in an MDP and induces a Markov chain $\M_\C$
that formalizes the behavior of $\M$ under control policy $\C$ \cite{Baier:PMC2008}.
In general, $\M_\C$ contains all the states in $S^+$ and hence may not be finite even though $\M$ is finite.
However, for a special case where $\C$ is memoryless, it can be shown that $\M_\C$ can be identified with a finite MC.

\begin{definition}
	Let $\M = (S, Act, \mathbf{P}, \i_{init}, \Pi, L)$ be a Markov decision process.
	A control policy $\C$ on $\M$ is \emph{memoryless} if and only if for each sequence
	$s_0s_1\ldots s_n$ and $t_0t_1\ldots t_m \in S^+$ with $s_n = t_m$,
	$\C(s_0s_1\ldots s_n) = \C(t_0t_1\ldots t_m)$.
	A memoryless control policy $\C$ can be described by a function $\C : S \to Act$.
\end{definition}

\section{Problem Formulation}
\label{sec:prob}
Consider a system that comprises the plant (e.g., the robot) and $N$ independent environment agents.
We assume that at any time instance, the state of the system, which incorporates the state of the plant and the environment agents,
can be precisely observed.
The system can regulate the state of the plant but
has no control over the state of the environment agents.
Hence, we do not distinguish between a control policy \emph{for the system} and a control policy \emph{for the plant}
and refer to them as a control policy in general, 
as there is no confusion that in both cases, only the state of the plant can be regulated and
both the system and the plant can precisely observe the current state of the complete system.
Hence, even though a control policy may be implemented on the plant,
it may be defined over the state of the complete system.

We assume that each environment agent can be modeled by a finite Markov chain.
Let $\M_i  = (S_i, \mathbf{P}_i, \i_{init,i}, \Pi_i, L_i)$ be the model of  the $i$th environment agent.
The plant is modeled either by a deterministic finite transition system or by a finite Markov decision process,
depending on whether each control action leads to a deterministic state transition.
We use $\T$ to denote the model of the plant and let
$\T = (S_0, Act, \longrightarrow, s_{init,0}, \Pi_0, L_0)$ for the case where $\T$ is a DFTS
and $\T = (S_0, Act, \mathbf{P}_0, \i_{init,0}, \Pi_0, L_0)$ for the case where $\T$ is an MDP.
For the simplicity of the presentation, we assume that for all $s \in S_0$, $Act(s) \not= \emptyset$.
In addition, we assume that all the components $\T, \M_1, \M_2, \ldots, \M_N$ in the system make a transition simultaneously, i.e., 
each of them makes a transition at every time step.

\begin{example}
\label{ex:ped-models}
Consider a problem where an autonomous vehicle (plant) needs to go through a pedestrian crossing while there are $N$
pedestrians (agents) who are already at or approaching the crossing.
Suppose the road is discretized into a finite number of cells $c_0, c_2, \ldots, c_M$.
The vehicle is modeled by either a DFTS $\T = (S_0, Act, \longrightarrow, s_{init,0}, \Pi_0, L_0)$
or an MDP $\T = (S_0, Act, \mathbf{P}_0, \i_{init,0}, \Pi_0, L_0)$ whose state $s \in S_0$ describes the cell occupied by the vehicle 
and whose action $\alpha \in Act$ corresponds to a motion primitive of the vehicle (e.g., stop, accelerate, decelerate).
If each motion primitive leads to a deterministic change in the vehicle's state, then $\T$ is a DFTS.
Otherwise, $\T$ is an MDP.
The motion of the $i$th pedestrian is modeled by an MC $\M_i  = (S_i, \mathbf{P}_i, \i_{init,i}, \Pi_i, L_i)$ whose
state $s \in S_i$ describes the cell occupied by the $i$th pedestrian.
The labeling function $L_i$, $i \in \{0, \ldots, N\}$ essentially maps each cell to its label,
indexed by the agent ID,
i.e., $L_i(c_j) = c_{j}^i$ for all $j \in \{0, \ldots M\}$.
\end{example}

\paragraph*{\textbf{Control Policy Synthesis Problem}}
Given a system model described by $\T, \M_1, \ldots, \M_N$ and a syntactically co-safe LTL formula 
$\varphi$ over $\Pi_0 \cup \Pi_1 \cup \ldots \cup \Pi_N$, 
we want to automatically synthesize a control policy that maximizes the probability for the system to satisfy $\varphi$.

\begin{example}
\label{ex:ped-spec}
Consider the autonomous vehicle problem described in Example \ref{ex:ped-models} and
the desired property stating that the vehicle does not collide with any pedestrian until
it reaches cell $c_M$ (e.g., the other side of the pedestrian crossing).
In this case, the specification $\varphi$ is given by
$\varphi = \left(\neg \bigvee_{i \geq 1, j \geq 0} (c^0_j \aand c^i_j)\right) \until c^0_M$.
Using simple logic manipulation, it can be checked that $\varphi$ is a co-safe LTL formula.
\end{example}

\section{Control Policy Synthesis}
\label{sec:syn}

We employ existing results in probabilistic verification and consider the following 3 main steps 
to solve the control policy synthesis problem defined in Section \ref{sec:prob}:
\begin{enumerate}
\item Compute the composition of all the system components  to obtain the complete system.
\item Construct the product MDP.
\item Extract an optimal control policy for the product MDP.
\end{enumerate}

In this section, we describe these steps in more detail and discuss their connection to our control policy synthesis problem
described in Section \ref{sec:prob}.

\subsection{Parallel Composition of System Components}
\label{ssec:comp}
Assuming that all the components of the system make a transition simultaneously,
we first construct the synchronous parallel composition of all the components to obtain the complete system.
Synchronous parallel composition of different types of components is defined as follows.

\begin{definition} 
	Let $\M_1 = (S_1, \mathbf{P}_1, \i_{init,1}, \Pi_1, L_1)$ and
	$\M_2 = (S_2, \mathbf{P}_2, \i_{init,2}, \Pi_2, L_2)$ be Markov chains.
	Their synchronous parallel composition, denoted by $\M_1 || \M_2$, 
	is the MC $\M = (S_1 \times S_2, \mathbf{P}, \i_{init}, \Pi_1 \cup \Pi_2, L)$ where:
	\begin{itemize}
	\item For each $s_1, s_1' \in S_1$ and $s_2, s_2' \in S_2$, 
	$\mathbf{P}(\langle s_1, s_2 \rangle, \langle s_1', s_2' \rangle) = \mathbf{P}_1(s_1, s_1')\mathbf{P}_2(s_2, s_2')$.
	\item For each $s_1 \in S_1$ and $s_2 \in S_2$, $\i_{init}(\langle s_1, s_2 \rangle) = \i_{init,1}(s_1)\i_{init,2}(s_2)$.
	\item For each $s_1 \in S_1$ and $s_2 \in S_2$, $L(\langle s_1, s_2 \rangle) = L(s_1) \cup L(s_2)$.
	\end{itemize}
\end{definition}

\begin{definition} 
	Let $\T_1 = (S_1, Act, \longrightarrow, s_{init}, \Pi_1, L_1)$ be a deterministic finite transition system and
	$\M_2 = (S_2, \mathbf{P}_2, \i_{init,2}, \Pi_2, L_2)$ be a Markov chain.
	Their synchronous parallel composition,
	denoted by $\T_1 || \M_2$, 
	is the MDP $\M = (S_1 \times S_2, Act, \mathbf{P}, \i_{init}, \Pi_1 \cup \Pi_2, L)$ where:
	\begin{itemize}
	\item For each $s_1, s_1' \in S_1$, $s_2, s_2' \in S_2$ and $\alpha \in Act$, 
	$\mathbf{P}(\langle s_1, s_2 \rangle, \alpha, \langle s_1', s_2' \rangle) = \mathbf{P}_2(s_2, s_2')$ if $s_1 \stackrel{\alpha}{\longrightarrow} s_1'$
	and $\mathbf{P}(\langle s_1, s_2 \rangle, \alpha, \langle s_1', s_2' \rangle) = 0$ otherwise.
	\item For each $s_2 \in S_2$, $\i_{init}(\langle s_{init}, s_2 \rangle) = \i_{init,2}(s_2)$ 
	and $\i_{init}(\langle s_1, s_2 \rangle) = 0$ for all $s_1 \in S \setminus \{s_{init}\}$.
	\item For each $s_1 \in S_1$ and $s_2 \in S_2$, $L(\langle s_1, s_2 \rangle) = L(s_1) \cup L(s_2)$.
	\end{itemize}
\end{definition}

\begin{definition} 
	Let $\M_1 = (S_1, Act, \mathbf{P}_1, \i_{init,1}, \Pi_1, L_1)$ be a Markov decision process and
	$\M_2 = (S_2, \mathbf{P}_2, \i_{init,2}, \Pi_2, L_2)$ be a Markov chain.
	Their synchronous parallel composition,
	denoted by $\M_1 || \M_2$, 
	is the MDP $\M = (S_1 \times S_2, Act, \mathbf{P}, \i_{init}, \Pi_1 \cup \Pi_2, L)$ where:
	\begin{itemize}
	\item For each $s_1, s_1' \in S_1$, $s_2, s_2' \in S_2$ and $\alpha \in Act$,
	$\mathbf{P}(\langle s_1, s_2 \rangle, \alpha, \langle s_1', s_2' \rangle) = \mathbf{P}_1(s_1, \alpha, s_1')\mathbf{P}_2(s_2, s_2')$.
	\item For each $s_1 \in S_1$ and $s_2 \in S_2$, $\i_{init}(\langle s_1, s_2 \rangle) = \i_{init,1}(s_1)\i_{init,2}(s_2)$.
	\item For each $s_1 \in S_1$ and $s_2 \in S_2$, $L(\langle s_1, s_2 \rangle) = L(s_1) \cup L(s_2)$.
	\end{itemize}
\end{definition}

From the above definitions, our complete system can be modeled by the MDP $\T || \M_1 || \ldots || \M_N$,
regardless of whether $\T$ is a DFTS or an MDP.
We denote this MDP by $\M = (S, Act, \mathbf{P}, \i_{init}, \Pi, L)$.

\subsection{Construction of Product MDP}
\label{ssec:pMDP}
Let $\A_\varphi = (Q, 2^\Pi, \delta, q_{init}, F)$ be a DFA that recognizes the good prefixes of $\varphi$.
Such $\A_\varphi$ can be automatically constructed using existing tools \cite{Latvala:EMC2003}.
Our next step is to obtain a finite MDP $\M_p = (S_p, Act_p, \mathbf{P}_p, \i_{p,init}, Q, L_p)$ as the product of $\M$ and $\A_\varphi$,
defined as follows.

\begin{definition} 
	Let $\M = (S, Act, \mathbf{P}, \i_{init}, \Pi, L)$ be an MDP and
	let $\A = (Q, 2^\Pi, \delta, q_{init}, F)$ be a DFA.
	The product of $\M$ and $\A$ is the MDP $\M_p = \M \otimes \A$ defined by%
	\footnote{We slightly modify the definition of atomic propositions and labeling function of the product MDP 
	from the definition often used in literature to facilitate incremental construction of product MDP,
	which is explained in Section \ref{ssec:inc_pMDP}.}
	$\M_p = (S_p, Act, \mathbf{P}_p, \i_{p,init}, \Pi, L_p)$
	where $S_p = S \times Q$ and $L_p(\langle s,q \rangle)= L(s)$.
	$\mathbf{P}_p$ is defined as
	\begin{equation}
		\mathbf{P}_p( \langle s,q \rangle, \alpha, \langle s',q' \rangle) = 
		\left\{ \begin{array}{l} \tilde{\mathbf{P}}_p( \langle s,q \rangle, \alpha, \langle s',q' \rangle) \\
		  \hspace{20mm}\hbox{if } q' = \delta(q, L(s'))\\
		0 \hspace{18mm}\hbox{otherwise} \end{array}\right.,
	\end{equation}
	where
	$\tilde{\mathbf{P}}_p( \langle s,q \rangle, \alpha, \langle s',q' \rangle) = \mathbf{P}(s, \alpha,s')$.
	For the rest of the paper, we refer to $\tilde{\mathbf{P}}_p: S_p \times Act \times S_p \to [0,1]$ as the 
	\emph{intermediate transition probability function} for $\M_p$.
	Finally,
	\begin{equation}
		\i_{p,init}(\langle s,q \rangle) = \left\{ \begin{array}{ll} \tilde{\i}_{p,init}(\langle s,q \rangle) &\hbox{if } q = \delta(q_{init}, L(s))\\ 
		0 &\hbox{otherwise} \end{array}\right.,
	\end{equation}
	where $\tilde{\i}_{p,init}(\langle s,q \rangle) =\i_{init}(s)$.
	For the rest of the paper, we refer to $\tilde{\i}_{p,init}: S_p \to [0,1]$ as the \emph{intermediate initial state distribution} for $\M_p$.
\end{definition}

Stepping through the above definition shows that 
given a path $r_{\M_p}^{\C_p} = \langle s_0,q_0 \rangle \langle s_1,q_1 \rangle \ldots$ on $\M_p$
generated under some control policy $\C_p$, 
the corresponding path $s_0s_1\ldots$ on $\M$ generates a word
$L(s_0) L(s_1) \ldots$ that satisfies $\varphi$ if and only if there exists $n \geq 0$ such that 
$q_n  \in F$ 
(and hence $q_0q_1\ldots q_n$ is an accepting run on $\A_\varphi$), 
in which case we say that $r_{\M_p}^{\C_p}$ is accepting.
Therefore, each accepting path of $\M_p$ uniquely corresponds to a path of $\M$
whose word satisfies $\varphi$.
In addition, a control policy $\C_p$ on $\M_p$ induces the corresponding control policy
$\C$ on $\M$.
The details for generating $\C$ from $\C_p$ can be found, e.g. in \cite{Baier:PMC2008,Ding:LTL2011}.

Based on this argument, our control policy synthesis problem defined in Section \ref{sec:prob} can be
reduced to computing a control policy for $\M_p$ that maximizes the probability of reaching 
a state in $B_p = \{ \langle s,q \rangle \in S_p \hspace{1mm}|\hspace{1mm} q \in F\}$.

\subsection{Control Policy Synthesis for Product MDP}
\label{ssec:policy}
For each $s \in S_p$, let $x_s$ denote the maximum probability of reaching a state in $B_p$,
starting from $s$.
Formall, $x_s = \sup_{\C} \prob^\C_{\M_p}(s \models \eventually B_p)$,
where, with an abuse of notation, $B_p$ in $\eventually B_p$ is a proposition that is satisfied by all states in $B_p$.
There are two main techniques for computing the probability $x_s$ for each $s \in S_p$:
linear programming (LP) and value iteration.
LP-based techniques yield an exact solution but
it typically does not scale as well as value iteration. 
On the other hand, value iteration is an iterative numerical technique.
This method works by successively computing the probability vector $(x_s^{(k)})_{s \in S_p}$ for increasing $k \geq 0$
such that $\lim_{k \to \infty} x_s^{(k)} = x_s$ for all $s \in S_p$.
Initially, we set $x_s^{(0)} = 1$ if $s \in B_p$ and $x_s^{(0)} = 0$ otherwise.
In the $(k+1)$th iteration where $k \geq 0$, we set
\begin{equation}
\label{eq:value_iteration}
  x_s^{(k+1)} = \left\{ \begin{array}{ll}
  1 &\hbox{if } s \in B_p\\
  \displaystyle{\max_{\alpha \in Act_p(s)}} \sum_{t \in S_p} \mathbf{P}_p(s, \alpha, t)x_t^{(k)} 
  &\hbox{otherwise}.
  \end{array}\right.
\end{equation}

In practice, we terminate the computation and say that $x_s^{(k)}$ converges when a termination criterion such as
$\max_{s \in S_p}|x_s^{(k+1)} - x_s^{(k)}| < \epsilon$ is satisfied for some fixed (typically very small) threshold $\epsilon$.

As discussed in \cite{Ciesinski:RTM2008,Kwiatkowska:IQV2011}, decomposition of $\M_p$ into
\emph{strongly connected components} (SCC) can help speed up value iteration.
$C \subseteq S_p$ is an SCC of $\M_p$ if there is a path in $\M_p$ between any two states in $C$
and $C$ is maximal (i.e., there does not exist any $\tilde{C} \subseteq S_p$ such that
$C \subset \tilde{C}$ and $\tilde{C}$ is an SCC).
The algorithm proposed in \cite{Tarjan:DFS1972} allows us to identify all the SCCs of $\M_p$
with time and space complexity that is linear in the size of $\M_p$.

The SCC-based value iteration works as follows.
First, we set $x_s^{(0)} = 1$ if $s \in B_p$ and $x_s^{(0)} = 0$ otherwise.%
\footnote{In the original algorithm, all the states $s \in S_p$ with $x_s = 1$ and 
all the states that cannot reach $B_p$ under any control policy need to be identified but
it has been shown in \cite{Kwiatkowska:IQV2011} that this step is not necessary for the correctness of the algorithm.}
Next, we identify all the SCCs $C_1^{\M_p}, \ldots, C_m^{\M_p}$ of $\M_p$.
From the definition of SCC, we get that $C_i^{\M_p} \cap C_j^{\M_p} = \emptyset, \forall i \not= j$
and $\bigcup_i C_i^{\M_p} = S_p$.
For each SCC $C_i^{\M_p}$, we define $Succ(C_i^{\M_p}) \subseteq S_p \setminus C_i^{\M_p}$
to be the set of all the immediate successors of states in $C_i^{\M_p}$ that are not in $C_i^{\M_p}$.
A (strict) partial order, $\prec_{\M_p}$, among $C_1^{\M_p}, \ldots, C_m^{\M_p}$ can be defined such that
$C_j^{\M_p} \prec_{\M_p} C_i^{\M_p}$ if $Succ(C_i^{\M_p}) \cap C_j^{\M_p} \not= \emptyset$.
(Note that from the definition of SCC and $Succ$, there cannot be cyclic dependency among SCCs;
hence, such a partial order can always be defined.)

An important property of SCCs and their partial order that 
we will exploit in the computation of the probability vector $(x_s)_{s \in S_p}$ is that 
the probability values of states in $C_i^{\M_p}$ can be affected only by
the probability values of states in $C_i^{\M_p}$ and all $C_j^{\M_p} \prec_{\M_p} C_i^{\M_p}$.
Thus, our next step is to generate an order $\mathbb{O}^{\M_p}$ among $C_1^{\M_p}, \ldots, C_m^{\M_p}$ such that
$C_i^{\M_p}$ appears before $C_j^{\M_p}$ in $\mathbb{O}^{\M_p}$ if $C_i^{\M_p} \prec_{\M_p} C_j^{\M_p}$.
We can then process each SCC separately, according to the order in $\mathbb{O}^{\M_p}$,
since the probability values of states in $C_j^{\M_p}$ that appears after $C_i^{\M_p}$ in $\mathbb{O}^{\M_p}$
cannot affect the probability values of states in $C_i^{\M_p}$.
Processing of SCC $C_i^{\M_p}$ terminates at the $k$th iteration where all $x_s^{(k)}$, $s \in C_i^{\M_p}$ converges.
Let $x_s$ be the value to which $x_s^{(k)}$ converges.
When processing $C_i^{\M_p}$, we exploit the order in $\mathbb{O}^{\M_p}$ and existing values of $x_{t}$ 
for all $t \in Succ(C_i^{\M_p})$ to determine the set of $s \in C_i^{\M_p}$ where $x_s^{(k+1)}$ needs to be updated from $x_s^{(k)}$.
The formula in (\ref{eq:value_iteration}) with $x_t^{(k)}$ replaced by $x_t$ for all $t \in Succ(C_i^{\M_p})$
can be used to update those $x_s^{(k+1)}$.
We refer the reader to \cite{Ciesinski:RTM2008,Kwiatkowska:IQV2011} for more details.

Note that computation of an order $\mathbb{O}^{\M_p}$ requires $O(|S_p|^2)$ time.
Thus, the pre-computation required by the SCC-based value iteration can be computationally expensive,
unless all the SCCs of $\M_p$ and an order $\mathbb{O}^{\M_p}$ are provided a-priori.
As a result, the SCC-based value iteration may require more computation time than the normal value iteration,
if the pre-computation time is also taken into account.

Once the vector $(x_s)_{s \in S_p}$ is computed, a memoryless control policy $\C$ such that for any $s \in S_p$, 
$\prob^\C_\M(s \models \eventually B_p) = x_s$ can be constructed as follows.
For each state $s \in S_p$, let $Act_p^{max}(s) \subseteq Act_p(s)$ be the set of actions such that
for all $\alpha \in Act_p^{max}(s)$,
$x_s = \sum_{t \in S_p} \mathbf{P}(s,\alpha,t)x_t$.
For each $s \in S_p$ with $x_s > 0$, let $\|s\|$ be the length of a shortest path from $s$ to a state in $B_p$, 
using only actions in $Act_p^{max}$. 
$\C(s) \in Act_p^{max}(s)$ for a state $s \in S_p \setminus B_p$ with $x_s > 0$
is then chosen such that $\mathbf{P}_p(s, \C(s), t) > 0$ for some $t \in S_p$ with $\|t\| = \|s\|-1$.
For a state $s \in S_p$ with $x_s = 0$ or a state $s \in B_p$, $\C(s) \in Act_p(s)$ can be chosen arbitrarily.

\section{Incremental Computation of Control Policies }
\label{sec:incremental}
Automatic synthesis described in the previous section suffers from the state explosion problem
as the composition of $\T$ and all $\M_1, \ldots, \M_N$ needs to be constructed, 
leading to an exponential blow up of the state space. 
In this section, we propose an incremental synthesis approach where
we progressively compute a sequence of control policies, 
taking into account only a small subset of the environment agents initially and
successively add more agents to the synthesis procedure in each iteration
until we hit the computational resource constraints.
Hence, even though the complete synthesis problem cannot be solved due to the computational resource limitation,
we can still obtain a reasonably good control policy.

\subsection{Overview of Incremental Computation of Control Policies}
\label{ssec:overview}
Initially, we consider a small subset
$\mathbf{M}_0 \subset \{\M_1, \ldots, \M_N\}$ of the environment agents.
For each $\M_i = (S_i, \mathbf{P}_i, \i_{init,i}, \Pi_i, L_i) \not\in \mathbf{M}_0$, 
we consider a simplified model $\tilde{\M}_i$ that essentially assumes that the $i$th environment agent is stationary
(i.e., we take into account their presence but do not consider their full model).
Formally, $\tilde{\M}_i = (\{s_i\}, \tilde{\mathbf{P}}_i, \tilde{\i}_{init,i}, \Pi_i, \tilde{L}_i)$ where
$s_i \in S_i$ can be chosen arbitrarily, 
$\tilde{\mathbf{P}}_i(s_i, s_i) = 1$, $\tilde{\i}_{init,i}(s_i) = 1$ and $\tilde{L}_i(s_i) = L_i(s_i)$.
Note that the choice of $s_i \in S_i$ may affect the performance of our incremental synthesis algorithm; hence, 
it should be chosen such that it is the most likely state of $\M_i$.
We let $\tilde{\mathbf{M}}_0 = \{\tilde{\M}_i \hspace{1mm}|\hspace{1mm} \M_i \in \{\M_1, \ldots, \M_N\} \setminus \mathbf{M}_0\}$.

The composition of $\T$, all $\M_i \in \mathbf{M}_0$ and all 
$\tilde{\M}_j \in \tilde{\mathbf{M}}_0$ is then constructed.
We let $\M^{\mathbf{M}_0}$ be the MDP that represents such composition.
Note that since $\tilde{\M}_i$ is typically smaller $\M_i$, 
$\M^{\mathbf{M}_0}$ is typically much smaller than the composition of $\T, \M_1, \ldots, \M_N$.
We identify all the SCCs of $\M^{\mathbf{M}_0}$ and their partial order.
Following the steps for synthesizing a control policy described in Section \ref{sec:syn},
we construct $\M^{\mathbf{M}_0}_p = \M^{\mathbf{M}_0} \otimes \A_\varphi$ 
where $\A_\varphi = (Q, 2^\Pi, \delta, q_{init}, F)$ is a DFA that recognizes the good prefixes of $\varphi$.
We also store the intermediate transition probability function and the intermediate initial state distribution for $\M^{\mathbf{M}_0}_p$
and denote these functions by $\tilde{\mathbf{P}}^{\mathbf{M}_{0}}_p$ and $\tilde{\i}^{\mathbf{M}_{0}}_{p,init}$, respectively.

At the end of the initialization period (i.e., the $0$th iteration), 
we obtain a control policy $\C^{\mathbf{M}_0}$ that maximizes the probability for $\M^{\mathbf{M}_0}$ to satisfy $\varphi$.
$\C^{\mathbf{M}_0}$ resolves all nondeterministic choices in $\M^{\mathbf{M}_0}$ 
and induces a Markov chain, which we denote by $\M^{\mathbf{M}_0}_{\C^{\mathbf{M}_0}}$.

Our algorithm then successively adds more full models of the rest of the environment agents
to the synthesis procedure at each iteration.
In the $(k+1)$th iteration where $k \geq 0$, we consider $\mathbf{M}_{k+1} = \mathbf{M}_k \cup \{\M_l\}$ for some 
$\M_l \in \{\M_1, \ldots, \M_N\} \setminus \mathbf{M}_k$.
Such $\M_l$ may be picked such that the probability for $\M^{\mathbf{M}_0}_{\C^{\mathbf{M}_0}} || \M_l$
to satisfy $\varphi$ is the minimum among all $\M_i \in \{\M_1, \ldots, \M_N\} \setminus \mathbf{M}_k$.
This probability can be efficiently computed using probabilistic verification \cite{Baier:PMC2008}.
(As an MC can be considered a special case of MDP with exactly one action enabled in each state,
we can easily adapt the techniques for computing the probability vector of a product MDP
described in Section \ref{ssec:policy} to compute the probability that $\M^{\mathbf{M}_0}_{\C^{\mathbf{M}_0}} || \M_l$ satisfies $\varphi$.)
We let $\tilde{\mathbf{M}}_{k+1} = \tilde{\mathbf{M}}_{k} \setminus \{\tilde{\M}_l\}$ and
let $\M^{\mathbf{M}_{k+1}}$ be the MDP that represents the composition of
$\T$, all $\M_i \in \mathbf{M}_{k+1}$ and all $\tilde{\M}_j \in \tilde{\mathbf{M}}_{k+1}$.
Next, we construct $\M^{\mathbf{M}_{k+1}}_p = \M^{\mathbf{M}_{k+1}} \otimes \A_\varphi$ and
obtain a control policy $\C^{\mathbf{M}_{k+1}}$ that maximizes the probability for $\M^{\mathbf{M}_{k+1}}$ to satisfy $\varphi$.
Similar to the initialization step,
during the construction of $\M^{\mathbf{M}_{k+1}}_p$, we store the intermediate transition probability function 
and the intermediate initial state distribution for $\M^{\mathbf{M}_{k+1}}_p$
and denote these functions by $\tilde{\mathbf{P}}^{\mathbf{M}_{k+1}}_p$ and $\tilde{\i}^{\mathbf{M}_{k+1}}_{p,init}$, respectively.

The process outlined in the previous paragraph terminates at the $K$th iteration where $\mathbf{M}_K = \{\M_1, \ldots, \M_N\}$
or when the computational resource constraints are exceeded.
To make this process more efficient, we avoid unnecessary computation and
exploit the objects computed in the previous iteration.
Consider an arbitrary iteration $k \geq 0$.
In Section \ref{ssec:inc_pMDP}, we show how $\M^{\mathbf{M}_{k+1}}_p$, $\tilde{\mathbf{P}}^{\mathbf{M}_{k+1}}_p$,
and $\tilde{\i}^{\mathbf{M}_{k+1}}_{p,init}$ 
can be incrementally constructed from $\M^{\mathbf{M}_{k}}_p$,
$\tilde{\mathbf{P}}^{\mathbf{M}_{k}}_p$ and $\tilde{\i}^{\mathbf{M}_{k}}_{p,init}$. 
Hence, we can avoid computing $\M^{\mathbf{M}_{k+1}}$.
In addition, as previously discussed in Section \ref{ssec:policy}, generating an order of SCCs can be computationally expensive.
Hence, we only compute the SCCs and their order for $\M^{\mathbf{M}_0}$ and all ${\M}_j \in \{\M_1, \ldots, \M_N\} \setminus \mathbf{M}_0$,
which are typically small.
Incremental construction of SCCs of $\M^{\mathbf{M}_{k+1}}$ and their order from those of 
$\M^{\mathbf{M}_{k}}$ is considered in Section \ref{ssec:inc_scc}.
(Note that we do not compute $\M^{\mathbf{M}_{k}}$ but only maintain its SCCs and their order,
which are incrementally constructed using the results from the previous iteration.)
Finally, Section \ref{ssec:inc_prob} describes computation of $\C^{\mathbf{M}_k}$,
using a method adapted from SCC-based value iteration where we avoid having to identify the SCCs of $\M^{\mathbf{M}_k}_p$ and their order.
Instead, we exploit the SCCs of $\M^{\mathbf{M}_k}$ and their order, 
which can be incrementally constructed using the approach described in Section \ref{ssec:inc_scc}.

\subsection{Incremental Construction of Product MDP}
\label{ssec:inc_pMDP}
For an iteration $k \geq 0$, let $\mathbf{M}_{k+1} = \mathbf{M}_k \cup \{\M_l\}$ for some 
$\M_l \in \{\M_1, \ldots, \M_N\} \setminus \mathbf{M}_k$.
In general, one can construct $\M^{\mathbf{M}_{k+1}}_p$ by first computing
$\M^{\mathbf{M}_{k+1}}$, which requires taking the composition of
a DFTS or an MDP with $N$ MCs, 
and then constructing $\M^{\mathbf{M}_{k+1}} \otimes A_{\varphi}$.
To accelerate the process of computing $\M^{\mathbf{M}_{k+1}}_p$,
we exploit the presence of $\M^{\mathbf{M}_{k}}_p$, 
its intermediate transition probability function $\tilde{\mathbf{P}}^{\mathbf{M}_{k}}_p$
and intermediate initial state distribution $\tilde{\i}^{\mathbf{M}_{k}}_{p,init}$, 
which are computed in the previous iteration.

First, note that a state $s_p$ of $\M^{\mathbf{M}_{k}}_p$ is of the form
$s_p = \langle s, q \rangle$ where $s = \langle s_0, s_1, \ldots, s_N \rangle \in S_0 \times S_1 \times \ldots \times S_N$
and $q \in Q$.
For $s = \langle s_0, s_1, \ldots, s_N \rangle \in S_0 \times S_1 \times \ldots \times S_N$, $i \in \{0, \ldots, N\}$
and $r \in S_i$,
we define $s|_{i \leftarrow r} = \langle s_0, \ldots, s_{i-1}, r, s_{i+1}, \ldots, s_N \rangle$, i.e.,
$s|_{i \leftarrow r}$ is obtained by replacing the $i$th element of 
$s$ by $r$.

\begin{lemma}
\label{lem:Mk+1}
Consider an arbitrary iteration $k \geq 0$.
Let $\mathbf{M}_{k+1} = \mathbf{M}_k \cup \{\M_l\}$ where $\M_l \in \{\M_1, \ldots, \M_N\} \setminus \mathbf{M}_k$.
Suppose $\M^{\mathbf{M}_{k}}_p = (S^{\mathbf{M}_{k}}_p, Act^{\mathbf{M}_{k}}_p, \mathbf{P}^{\mathbf{M}_{k}}_p, 
\i_{p,init}^{\mathbf{M}_{k}}, \Pi^{\mathbf{M}_{k}}_p, L^{\mathbf{M}_{k}}_p)$ and
$\M_l = (S_l, \mathbf{P}_l, \i_{init,l}, \Pi_l, L_l)$.
Assuming that for any $i,j \in \{0, \ldots, N\}$, $\Pi_i \cap \Pi_j = \emptyset$,
then $\M^{\mathbf{M}_{k+1}}_p = (S^{\mathbf{M}_{k+1}}_p, Act^{\mathbf{M}_{k+1}}_p, \mathbf{P}^{\mathbf{M}_{k+1}}_p, 
\i_{p,init}^{\mathbf{M}_{k+1}}, \Pi^{\mathbf{M}_{k}}_p, L^{\mathbf{M}_{k+1}}_p)$ where
$S^{\mathbf{M}_{k+1}}_p = \{ \langle s|_{l \leftarrow r}, q \rangle \hspace{1mm}|\hspace{1mm} 
\langle s, q \rangle \in S^{\mathbf{M}_k}_p \hbox{ and } r \in S_l\}$,
$Act^{\mathbf{M}_{k+1}}_p = Act^{\mathbf{M}_{k}}_p$,
$\Pi^{\mathbf{M}_{k+1}}_p = \Pi^{\mathbf{M}_{k}}_p$,
and for any $s = \langle s_0, \ldots, s_N \rangle, s' = \langle s_0', \ldots, s_N' \rangle \in S_0 \times \ldots S_N$
and $q, q' \in Q$,
\begin{itemize}	
	\item $\mathbf{P}^{\mathbf{M}_{k+1}}_p( \langle s, q \rangle, \alpha, \langle s', q' \rangle) = 
	  \left\{ \begin{array}{ll}  \tilde{\mathbf{P}}^{\mathbf{M}_{k+1}}_p( \langle s, q \rangle, \alpha, \langle s', q' \rangle)\\
	  &\hspace{-10mm}\hbox{if } q' = \delta(q, L^{\mathbf{M}_{k+1}}_p(\langle s', q' \rangle))\\0 &\hspace{-10mm}\hbox{otherwise} \end{array}\right.$,
	  where the intermediate transition probability function is given by
	  \begin{equation}
	    \label{eq:inc_P}
	    \begin{array}{l}
	    \tilde{\mathbf{P}}^{\mathbf{M}_{k+1}}_p( \langle s, q \rangle, \alpha, \langle s', q' \rangle) = \\
	      \hspace{10mm}
	      \mathbf{P}_l(s_l, s_l')
	      \tilde{\mathbf{P}}^{\mathbf{M}_k}_p( \langle \tilde{s}, q \rangle, \alpha, \langle \tilde{s}', q' \rangle)
	      \end{array}
	  \end{equation}
	 for any $\langle \tilde{s}, q \rangle, \langle \tilde{s}', q' \rangle \in S^{\mathbf{M}_{k}}_p$
	 such that $\tilde{s}|_{l \leftarrow s_l} = s$ and $\tilde{s}'|_{l \leftarrow s_l'} = s'$, 
	 
	\item $\i_{p,init}^{\mathbf{M}_{k+1}}(\langle s, q \rangle) = 
	  \left\{ \begin{array}{ll} \tilde{\i}_{p,init}^{\mathbf{M}_{k+1}}(\langle s, q \rangle)\\
	  &\hspace{-15mm}\hbox{if } q = \delta(q_{init}, L^{\mathbf{M}_{k+1}}_p(\langle s, q \rangle))\\ 0 &\hspace{-15mm}\hbox{otherwise} \end{array}\right.$
	  where the intermediate initial state distribution is given by
	  \begin{equation}
	    \label{eq:inc_i}
	    \tilde{\i}_{p,init}^{\mathbf{M}_{k+1}}(\langle s, q \rangle) = 
	      \i_{init,l}(s_l)
	      \tilde{\i}_{p,init}^{\mathbf{M}_{k}}(\langle \tilde{s}, q \rangle)
	  \end{equation}
	  for any $\langle \tilde{s}, q \rangle \in S^{\mathbf{M}_{k}}_p$
	  such that $\tilde{s}|_{l \leftarrow s_l} = s$, and 
	  
	\item $L^{\mathbf{M}_{k+1}}_p(\langle s, q \rangle) = \big(L^{\mathbf{M}_{k}}_p(\langle \tilde{s}, q \rangle) \setminus L_l(\tilde{s}_l) \big) \cup L_l(s_l)$
	  for any $\langle \tilde{s}, q \rangle \in S^{\mathbf{M}_{k}}_p$ such that $\tilde{s}|_{l \leftarrow s_l} = s$. 
\end{itemize}
\end{lemma}
\begin{proof}
The correctness of $S^{\mathbf{M}_{k+1}}_p$, $Act^{\mathbf{M}_{k+1}}_p$, $\Pi^{\mathbf{M}_{k+1}}_p$ and $L^{\mathbf{M}_{k+1}}_p$
is straightforward to verify.
Hence, we will only provide the proof for the correctness of $\mathbf{P}^{\mathbf{M}_{k+1}}_p$ and $\tilde{\mathbf{P}}^{\mathbf{M}_{k+1}}_p$.
The correctness of $\i_{p,init}^{\mathbf{M}_{k+1}}$ and $\tilde{\i}_{p,init}^{\mathbf{M}_{k+1}}$ can be proved in a similar way.

Consider an arbitrary iteration $k \geq 0$ and let $\M^{\mathbf{M}_{k}} = (S^{\mathbf{M}_{k}}, Act^{\mathbf{M}_{k}}, \mathbf{P}^{\mathbf{M}_{k}}, 
\i_{init}^{\mathbf{M}_{k}}, \Pi^{\mathbf{M}_{k}}, L^{\mathbf{M}_{k}})$
and $\M^{\mathbf{M}_{k+1}} = (S^{\mathbf{M}_{k+1}}, Act^{\mathbf{M}_{k+1}}, \mathbf{P}^{\mathbf{M}_{k+1}}, 
\i_{init}^{\mathbf{M}_{k+1}}, \Pi^{\mathbf{M}_{k+1}}, L^{\mathbf{M}_{k+1}})$.
It is obvious from the definition of product MDP that $\mathbf{P}^{\mathbf{M}_{k+1}}_p$ is correct
as long as $\tilde{\mathbf{P}}^{\mathbf{M}_{k+1}}_p$ is correct, i.e.,
$\tilde{\mathbf{P}}^{\mathbf{M}_{k+1}}_p(\langle s, q \rangle, \alpha, \langle s', q' \rangle) = \mathbf{P}^{\mathbf{M}_{k+1}}(s, \alpha, s')$ 
for all $\langle s, q \rangle, \langle s', q' \rangle \in S^{\mathbf{M}_{k+1}}_p$
and $\alpha \in Act^{\mathbf{M}_{k+1}}_p$.
Hence, we only need to prove the correctness of $\tilde{\mathbf{P}}^{\mathbf{M}_{k+1}}_p$.

Assume that $\tilde{\mathbf{P}}^{\mathbf{M}_{k}}_p$ is correct, i.e., 
$\tilde{\mathbf{P}}^{\mathbf{M}_{k}}_p(\langle s, q \rangle, \alpha, \langle s', q' \rangle) = \mathbf{P}^{\mathbf{M}_{k}}(s, \alpha, s')$ for all
$\langle s, q \rangle, \langle s', q' \rangle \in S^{\mathbf{M}_{k}}_p$ and $\alpha \in Act^{\mathbf{M}_{k}}_p$.
Let $l$ be the index such that $\mathbf{M}_{k+1} = \mathbf{M}_k \cup \{\M_l\}$.
Consider arbitrary $\langle s, q \rangle, \langle s', q' \rangle \in S^{\mathbf{M}_{k+1}}_p$ and $\alpha \in Act^{\mathbf{M}_{k+1}}_p$.
Suppose $s = \langle s_0, \ldots, s_N \rangle$ and $s' = \langle s_0', \ldots, s_N' \rangle$.
Note that since $\tilde{\M}_l$ only contains one state,
there exists exactly one $\langle \tilde{s}, q \rangle \in S^{\mathbf{M}_{k}}_p$ 
and exactly one $\langle \tilde{s}', q' \rangle \in S^{\mathbf{M}_{k}}_p$
such that $\tilde{s}|_{l \leftarrow s_l} = s$ and $\tilde{s}'|_{l \leftarrow s_l'} = s'$.
Since $\mathbf{M}_k$ is the composition of $\T$, all $\M_i \in \mathbf{M}_k$ and all 
$\tilde{\M}_j \in \tilde{\mathbf{M}}_k$
and since $\M_l \not\in \mathbf{M}_k$ and $\tilde{\mathbf{P}}_l(\cdot, \cdot) = 1$, 
it follows that if $\T$ is a DFTS, then
\begin{equation*}
	      \mathbf{P}^{\mathbf{M}_k}( \tilde{s}, \alpha, \tilde{s}')
  = \left\{ \begin{array}{ll}
     \displaystyle{\prod_{i \in \{1,\ldots,N\} \setminus \{l\}}} \mathbf{P}_i(s_i, s_i') &\hbox{if } s_0 \stackrel{\alpha}{\longrightarrow} s_0'\\
     0 &\hbox{otherwise}
  \end{array}\right.,
\end{equation*}
and if $\T$ is an MDP, then
\begin{equation*}
	      \mathbf{P}^{\mathbf{M}_k}( \tilde{s}, \alpha, \tilde{s}')
  = 
     \mathbf{P}_0(s_0, \alpha, s_0')
     \displaystyle{\prod_{i \in \{1,\ldots,N\} \setminus \{l\}}} \mathbf{P}_i(s_i, s_i').
\end{equation*}

Thus, $\mathbf{P}^{\mathbf{M}_{k+1}}(s, \alpha, s') = \mathbf{P}_l(s_l, s_l')
\mathbf{P}^{\mathbf{M}_{k}}(\tilde{s}, \alpha, \tilde{s}')$.
Combining this with (\ref{eq:inc_P}), we get
\begin{eqnarray*}
&&\tilde{\mathbf{P}}^{\mathbf{M}_{k+1}}_p( \langle s, q \rangle, \alpha, \langle s', q' \rangle) \\
  &&\hspace{20mm}= \mathbf{P}_l(s_l, s_l')
           \tilde{\mathbf{P}}^{\mathbf{M}_k}_p( \langle \tilde{s}, q \rangle, \alpha, \langle \tilde{s}', q' \rangle)\\
  &&\hspace{20mm}= \mathbf{P}_l(s_l, s_l')
           \mathbf{P}^{\mathbf{M}_{k}}(\tilde{s}, \alpha, \tilde{s}')\\
  &&\hspace{20mm}= \mathbf{P}^{\mathbf{M}_{k+1}}(s, \alpha, s').
\end{eqnarray*}

By definition, we can conclude that $\tilde{\mathbf{P}}^{\mathbf{M}_{k+1}}_p$ is correct.
\end{proof}

\subsection{Incremental Construction of SCCs}
\label{ssec:inc_scc}
Consider an arbitrary iteration $k \geq 0$. 
Let $l$ be the index of the environment agent such that 
$\mathbf{M}_{k+1} = \mathbf{M}_k \cup \{\M_l\}$.
In this section, we first provide a way to incrementally identify all the SCCs of $\M^{\mathbf{M}_{k+1}}$ 
from all the SCCs of $\M^{\mathbf{M}_{k}}$ and $\M_l$.
We conclude the section with incremental construction of the partial order over the SCCs of $\M^{\mathbf{M}_{k+1}}$
from the partial order defined over the SCCs of $\M^{\mathbf{M}_{k}}$ and $\M_l$. 

\begin{lemma}
\label{lem:SCC-Mk+1}
Let $C^{\mathbf{M}_{k}}$ be an SCC of $\M^{\mathbf{M}_{k}}$
and $C^l$ be an SCC of $\M_l$ where $\mathbf{M}_{k+1} = \mathbf{M}_k \cup \{\M_l\}$.
Suppose either of the following conditions holds:

\textbf{Cond 1:} $|C^{\mathbf{M}_{k}}| = 1$ and the state in $C^{\mathbf{M}_{k}}$ does not have a self-loop in $\M^{\mathbf{M}_{k}}$.

\textbf{Cond 2:} $|C^l| = 1$ and the state in $C^l$ does not have a self-loop in $\M_l$.

Then, for any $s \in C^{\mathbf{M}_{k}}$ and $r \in C^l$, 
$\{s|_{l \leftarrow r}\}$ is an SCC of $\M^{\mathbf{M}_{k+1}}$.
Otherwise,
$\{s|_{l \leftarrow r} \hspace{1mm}|\hspace{1mm} s \in C^{\mathbf{M}_{k}}, r \in C^l\}$
is an SCC of $\M^{\mathbf{M}_{k+1}}$.
\end{lemma}
\begin{proof}
First, we consider the case where Cond 1 or Cond 2 holds and
consider arbitrary $s \in C^{\mathbf{M}_{k}}$ and $r \in C^l$.
To show that $\{s|_{l \leftarrow r}\}$ is an SCC of $\M^{\mathbf{M}_{k+1}}$,
we will show that there is no path from $s|_{l \leftarrow r}$ to itself in $\M^{\mathbf{M}_{k+1}}$.
Since condition (1) or condition (2) holds, either there is no path from $s$ to itself in $\M^{\mathbf{M}_{k}}$
or there is no path from $r$ to itself in $C^l$.
Assume, by contradiction, that there is a path from $s|_{l \leftarrow r}$ to itself in $\M^{\mathbf{M}_{k+1}}$.
Let this path be $s|_{l \leftarrow r}, s^1, s^2, \ldots, s^n, s|_{l \leftarrow r}$ where 
for each $i \in \{1, \ldots, n\}$, $s^i = \langle s^i_0, \ldots, s^i_N \rangle$.
From the proof of Lemma \ref{lem:Mk+1}, we get that
$\mathbf{P}^{\M_{k+1}}(s|_{l \leftarrow r}, \alpha, s^1) = \mathbf{P}_l(r, s^1_l)\mathbf{P}^{\M_{k}}(s, \alpha, \tilde{s}^1)$,
$\mathbf{P}^{\M_{k+1}}(s^n, \alpha, s|_{l \leftarrow r}) = \mathbf{P}_l(s^n_l, r)\mathbf{P}^{\M_{k}}(\tilde{s}^n, \alpha, s)$
and $\mathbf{P}^{\M_{k+1}}(s^i, \alpha, s^{i+1}) = \mathbf{P}_l(s^i_l, s^{i+1}_l)\mathbf{P}^{\M_{k}}(\tilde{s}^i, \alpha, \tilde{s}^{i+1})$
for all $\alpha \in Act^{\mathbf{M}_{k+1}}$ where for each $i \in \{1, \ldots, n\}$,
$\tilde{s}^i \in S^{\mathbf{M}_{k}}$ such that $\tilde{s}^i|_{l \leftarrow s^i_l} = s^i$.

Since $s|_{l \leftarrow r}, s^1, s^2, \ldots, s^n, s|_{l \leftarrow r}$ is a path in $\M^{\mathbf{M}_{k+1}}$,
there exist $\alpha_0, \ldots, \alpha_n \in Act^{\mathbf{M}_{k+1}}$ such that
$\mathbf{P}^{\M_{k+1}}(s|_{l \leftarrow r}, \alpha_0, s^1)$,
$\mathbf{P}^{\M_{k+1}}(s^n, \alpha_n, s|_{l \leftarrow r})$,
$\mathbf{P}^{\M_{k+1}}(s^i, \alpha_i, s^{i+1}) > 0$
for all $i \in \{1, \ldots, n\}$.
Thus, it must be the case that 
$\mathbf{P}_l(r, s^1_l)$, $\mathbf{P}_l(s^n_l, r)$, $\mathbf{P}_l(s^i_l, s^{i+1}_l) > 0$
and $\mathbf{P}^{\M_{k}}(s, \alpha, \tilde{s}^1)$, $\mathbf{P}^{\M_{k}}(\tilde{s}^n, \alpha, s)$, 
$\mathbf{P}^{\M_{k}}(\tilde{s}^i, \alpha, \tilde{s}^{i+1}) > 0$
for all $i \in \{1, \ldots, n\}$.
But then, $r, s^1_l, \ldots, s^n_l, r$ is a path in $C^l$
and $s, \tilde{s}^1, \ldots, \tilde{s}^n, s$ is a path in $\M^{\mathbf{M}_{k}}$,
leading to a contradiction.

Next, consider the case where both Cond 1 and Cond 2 do not hold.
To show that $C^{\mathbf{M}_{k+1}} = \{s|_{l \leftarrow r} \hspace{1mm}|\hspace{1mm} s \in C^{\mathbf{M}_{k}}, r \in C^l\}$ 
is an SCC of $\M^{\mathbf{M}_{k+1}}$,
we need to show that for any $s, \tilde{s} \in C^{\mathbf{M}_{k+1}}$ 
and any $s' \notin C^{\mathbf{M}_{k+1}}$,
(1) there is a path in $\M^{\mathbf{M}_{k+1}}$ from $s$ to $\tilde{s}$, and
(2) there is no path in $\M^{\mathbf{M}_{k+1}}$ either from $s$ to $s'$ or from $s'$ to $s$.
Both of these statements can be proved by contradiction, using the same reasoning as
in the proof above for the case where either Cond 1 or Cond 2 holds.
\end{proof}

We say that an SCC $C^{\mathbf{M}_{k+1}}$ of $\M^{\mathbf{M}_{k+1}}$
is \emph{derived} from $\langle C^{\mathbf{M}_{k}}, C^l \rangle$, where
$C^{\mathbf{M}_{k}}$ is an SCC of $\M^{\mathbf{M}_{k}}$ 
and $C^l$ is an SCC of $\M_l$,
if $C^{\mathbf{M}_{k+1}}$ is constructed from $C^{\mathbf{M}_{k}}$
and $C^l$ according to Lemma \ref{lem:SCC-Mk+1}, i.e.,
$C^{\mathbf{M}_{k+1}} = \{s|_{l \leftarrow r}\}$ for some $s \in C^{\mathbf{M}_{k}}$ and $r \in C^l$
if Cond 1 or Cond 2 in Lemma \ref{lem:SCC-Mk+1} holds;
otherwise, $C^{\mathbf{M}_{k+1}} = \{s|_{l \leftarrow r} \hspace{1mm}|\hspace{1mm} s \in C^{\mathbf{M}_{k}}, r \in C^l\}$.

\begin{lemma}
\label{lem:SCC-Mk+1:unique}
For each SCC $C^{\mathbf{M}_{k+1}}$ of $\M^{\mathbf{M}_{k+1}}$,
there exists a unique $\langle C^{\mathbf{M}_{k}}, C^l \rangle$ from which $C^{\mathbf{M}_{k+1}}$ is derived.
\end{lemma}
\begin{proof}
Similar to Lemma \ref{lem:Mk+1}, it can be checked that
$S^{\mathbf{M}_{k+1}} = \{s|_{l \leftarrow r} \hspace{1mm}|\hspace{1mm} s \in S^{\mathbf{M}_k} \hbox{ and } r \in S_l\}$ 
is the set of states of $\M^{\mathbf{M}_{k+1}}$.
Consider an arbitrary SCC $C^{\mathbf{M}_{k+1}}$ of $\M^{\mathbf{M}_{k+1}}$ and
an arbitrary $s = \langle s_0, \ldots, s_N \rangle \in C^{\mathbf{M}_{k+1}}$.

By definition, for any arbitrary SCC $C^{\mathbf{M}_{k}}$ of $\M^{\mathbf{M}_{k}}$
and arbitrary SCC $C^l$ of $\M_l$,
$C^{\mathbf{M}_{k+1}}$ is derived from $\langle C^{\mathbf{M}_{k}}, C^l \rangle$ 
only if $s_l \in C^l$ and there exist $s' \in C^{\mathbf{M}_{k}}$ such that $s'|_{l \leftarrow s_l} = s$.
But since $\tilde{\M}_l$ contains exactly one state, 
there exists a unique $s' \in S^{\mathbf{M}_{k}}$ such that $s'|_{l \leftarrow s_l} = s$.
Also, from the definition of SCC, 
there exist a unique SCC $C^{\mathbf{M}_{k}}$ of $\M^{\mathbf{M}_{k}}$ and a unique
SCC $C^l$ of $\M_l$ such that $s' \in C^{\mathbf{M}_{k}}$ and $s_l \in C^l$.
Thus, it cannot be the case that $C^{\mathbf{M}_{k+1}}$ is derived from 
$\langle \tilde{C}^{\mathbf{M}_{k}}, \tilde{C}^l \rangle$ where
$\tilde{C}^{\mathbf{M}_{k}} \not= C^{\mathbf{M}_{k}}$ or $\tilde{C}^l \not= C^l$.
Applying Lemma \ref{lem:SCC-Mk+1}, we get that there exists an SCC $\tilde{C}^{\mathbf{M}_{k+1}}$ of $\M^{\mathbf{M}_{k+1}}$
that is derived from $\langle C^{\mathbf{M}_{k}}, C^l \rangle$ and contains $s$.
Since $s \in C^{\mathbf{M}_{k+1}}$ and $s \in \tilde{C}^{\mathbf{M}_{k+1}}$, from the definition of SCC, it must be the case that
$C^{\mathbf{M}_{k+1}} = \tilde{C}^{\mathbf{M}_{k+1}}$;
thus, $C^{\mathbf{M}_{k+1}}$ must be derived from $\langle C^{\mathbf{M}_{k}}, C^l \rangle$.
\end{proof}

Lemma \ref{lem:SCC-Mk+1} and Lemma \ref{lem:SCC-Mk+1:unique} provide a way to generate all the SCCs of $\M^{\mathbf{M}_{k+1}}$
from all the SCCs of $\M^{\mathbf{M}_{k}}$ and $\M_l$ as formally stated below.

\begin{corollary}
The set of all the SCCs of $\M^{\mathbf{M}_{k+1}}$ is given by
\begin{equation*}
\begin{array}{l}
\big\{C^{\mathbf{M}_{k+1}} \hbox{ derived from } \langle C^{\mathbf{M}_{k}}, C^l \rangle \hspace{1mm}|\hspace{1mm}\\
\hspace{6mm}C^{\mathbf{M}_{k}} \hbox{ is an SCC of } \M^{\mathbf{M}_{k}} \hbox{ and } C^l \hbox{ is an SCC of } \M_l\big\}.
\end{array}
\end{equation*}
\end{corollary}

Finally, in the following lemma, we provide a necessary condition,
based on the partial order over the SCCs of $\M^{\mathbf{M}_{k}}$ and $\M_l$,
for the existence of the partial order between two SCCs of $\M^{\mathbf{M}_{k+1}}$. 

\begin{lemma}
\label{lem:SCC-Mk+1:order}
Let $C^{\mathbf{M}_{k+1}}_1$ and $C^{\mathbf{M}_{k+1}}_2$ be SCCs of $\M^{\mathbf{M}_{k+1}}$.
Suppose $C^{\mathbf{M}_{k+1}}_1$ is derived from $\langle C^{\mathbf{M}_{k}}_1, C^l_1 \rangle$
and $C^{\mathbf{M}_{k+1}}_2$ is derived from $\langle C^{\mathbf{M}_{k}}_2, C^l_2 \rangle$ 
where $C^{\mathbf{M}_{k}}_1$ and $C^{\mathbf{M}_{k}}_2$ are SCCs of $\M^{\mathbf{M}_{k}}$ 
and $C^l_1$ and $C^l_2$ are SCCs of $\M_l$.
Then, $C^{\mathbf{M}_{k+1}}_1 \prec_{\M^{\mathbf{M}_{k+1}}} C^{\mathbf{M}_{k+1}}_2$ only if
$C^{\mathbf{M}_{k}}_1 \prec_{\M^{\mathbf{M}_{k}}} C^{\mathbf{M}_{k}}_2$ and
$C^l_1 \prec_{\M_l} C^l_2$.
\end{lemma}
\begin{proof}
Consider the case where $C^{\mathbf{M}_{k+1}}_1 \prec_{\M^{\mathbf{M}_{k+1}}} C^{\mathbf{M}_{k+1}}_2$.
By definition, $Succ(C^{\mathbf{M}_{k+1}}_2) \cap C^{\mathbf{M}_{k+1}}_1 \not= \emptyset$.
Consider a state $s' = \langle s_0', \ldots, s_N' \rangle \in Succ(C^{\mathbf{M}_{k+1}}_2) \cap C^{\mathbf{M}_{k+1}}_1$.
Since $s' \in Succ(C^{\mathbf{M}_{k+1}}_2)$, there exists $s = \langle s_0, \ldots, s_N \rangle \in C^{\mathbf{M}_{k+1}}_2$
and $\alpha \in Act^{\mathbf{M}_{k+1}}$ such that $\mathbf{P}^{\mathbf{M}_{k+1}}(s, \alpha, s') > 0$.
But from the proof of Lemma \ref{lem:Mk+1}, 
$\mathbf{P}^{\mathbf{M}_{k+1}}(s, \alpha, s') = \mathbf{P}_l(s_l, s_l') \mathbf{P}^{\mathbf{M}_{k}}(\tilde{s}, \alpha, \tilde{s}')$
where $\tilde{s}$ and $\tilde{s}'$ are unique states in $S^{\mathbf{M}_{k}}$ such that
$\tilde{s}|_{l \leftarrow s_l} = s$ and $\tilde{s}'|_{l \leftarrow s_l'} = s'$.
Thus, it must be the case that $\mathbf{P}_l(s_l, s_l') > 0$ and
$\mathbf{P}^{\mathbf{M}_{k}}(\tilde{s}, \alpha, \tilde{s}') > 0$.
In addition, since  $C^{\mathbf{M}_{k+1}}_1$ is derived from $\langle C^{\mathbf{M}_{k}}_1, C^l_1 \rangle$
and $C^{\mathbf{M}_{k+1}}_2$ is derived from $\langle C^{\mathbf{M}_{k}}_2, C^l_2 \rangle$,
from Lemma \ref{lem:SCC-Mk+1} and Lemma \ref{lem:SCC-Mk+1:unique},
it must be the case that $\tilde{s} \in C^{\mathbf{M}_{k}}_2$, $\tilde{s}' \in C^{\mathbf{M}_{k}}_1$,
$s_l \in C^l_2$ and $s_l' \in C^l_1$.
Since $\tilde{s} \in C^{\mathbf{M}_{k}}_2$, $\tilde{s}' \in C^{\mathbf{M}_{k}}_1$
and $\mathbf{P}^{\mathbf{M}_{k}}(\tilde{s}, \alpha, \tilde{s}') > 0$, 
we can conclude that $\tilde{s}' \in Succ(C^{\mathbf{M}_{k}}_2) \cap C^{\mathbf{M}_{k}}_1$,
and therefore, by definition, $C^{\mathbf{M}_{k}}_1 \prec_{\M^{\mathbf{M}_{k}}} C^{\mathbf{M}_{k}}_2$.
Similarly, since $s_l \in C^l_2$, $s_l' \in C^l_1$ and $\mathbf{P}_l(s_l, s_l') > 0$,
we can conclude that $s_l' \in Succ(C^l_2) \cap C^l_1$,
and therefore, by definition, $C^l_1 \prec_{\M_l} C^l_2$.
\end{proof}

\subsection{Computation of Probability Vector and Control Policy for $\M^{\mathbf{M}_{k}}_p$ from SCCs of $\M^{\mathbf{M}_{k}}$}
\label{ssec:inc_prob}

Consider an arbitrary iteration $k \geq 0$ and the associated product MDP
$\M^{\mathbf{M}_{k}}_p = (S^{\mathbf{M}_{k}}_p, Act^{\mathbf{M}_{k}}_p, \mathbf{P}^{\mathbf{M}_{k}}_p, 
\i_{p,init}^{\mathbf{M}_{k}}, \Pi^{\mathbf{M}_{k}}_p, L^{\mathbf{M}_{k}}_p)$.
Similar to the SCC-based value iteration, we want to generate a partition
$\{D^{\mathbf{M}_{k}}_{p,1}, \ldots, D^{\mathbf{M}_{k}}_{p,m_{k}}\}$ of $S^{\mathbf{M}_{k}}_p$
with a partial order $\prec_{\M^{\mathbf{M}_{k}}_p}$
such that $D^{\mathbf{M}_{k}}_{p,j} \prec_{\M^{\mathbf{M}_{k}}_p} D^{\mathbf{M}_{k+1}}_{p,i}$
if $Succ(D^{\mathbf{M}_{k}}_{p,i}) \cap D^{\mathbf{M}_{k}}_{p,j} \not= \emptyset$.
However, we relax the condition that each $D^{\mathbf{M}_{k}}_{p,i}, i \in \{1, \ldots, m_{k}\}$
is an SCC of $\M^{\mathbf{M}_{k}}_p$ and only require that
if $D^{\mathbf{M}_{k}}_{p,i}$ contains a state in an SCC $C^{\mathbf{M}_{k}}_p$ of $\M^{\mathbf{M}_{k}}_p$,
then it has to contain all the states in $C^{\mathbf{M}_{k}}_p$.
Hence, $D^{\mathbf{M}_{k}}_i$ may include all the states in multiple SCCs of $\M^{\mathbf{M}_{k}}_p$.
The following lemmas provide a method for constructing
$\{D^{\mathbf{M}_{k}}_1, \ldots, D^{\mathbf{M}_{k}}_{m_{k}}\}$ and their partial order 
from SCCs of $\M^{\mathbf{M}_{k}}$ and their partial order,
which can be incrementally constructed as described in Section \ref{ssec:inc_scc}.

\begin{lemma}
\label{lem:Ck}
Let $C^{\mathbf{M}_{k}}_p$ be an SCC of $\M^{\mathbf{M}_{k}}_p$.
Then, there exists a unique SCC $C^{\mathbf{M}_{k}}$ of $\M^{\mathbf{M}_{k}}$ 
such that $C^{\mathbf{M}_{k}}_p \subseteq C^{\mathbf{M}_{k}} \times Q$.
\end{lemma}
\begin{proof}
This follows from the definition of product MDP that for any $s,s' \in S^{\mathbf{M}_{k}}$
and $q,q' \in Q$, there is a path from $\langle s, q \rangle$ to $\langle s', q' \rangle$ in $\M^{\mathbf{M}_{k}}_p$
only if there is a path from $s$ to $s'$ in $\M^{\mathbf{M}_{k}}$.
\end{proof}

\begin{lemma}
\label{lem:Ck-order}
Let $C^{\mathbf{M}_{k}}_p$ and $\tilde{C}^{\mathbf{M}_{k}}_p$ be SCCs of $\M^{\mathbf{M}_{k}}_p$.
Suppose $C^{\mathbf{M}_{k}}$ and $\tilde{C}^{\mathbf{M}_{k}}$ are unique SCCs of $\M^{\mathbf{M}_{k}}$
such that $C^{\mathbf{M}_{k}}_p \subseteq C^{\mathbf{M}_{k}} \times Q$ and
$\tilde{C}^{\mathbf{M}_{k}}_p \subseteq \tilde{C}^{\mathbf{M}_{k}} \times Q$.
Then, $C^{\mathbf{M}_{k}}_p \prec_{\M^{\mathbf{M}_{k}}_p} \tilde{C}^{\mathbf{M}_{k}}_p$
only if $C^{\mathbf{M}_{k}} \prec_{\M^{\mathbf{M}_{k}}} \tilde{C}^{\mathbf{M}_{k}}$.
\end{lemma}
\begin{proof}
This follows from the definition of product MDP since for any $\langle s, q \rangle \in S^{\mathbf{M}_{k}}_p$,
$\langle \tilde{s}, \tilde{q} \rangle \in S^{\mathbf{M}_{k}}_p$ is a successor of $\langle s, q \rangle$ in $\M^{\mathbf{M}_{k}}_p$
only if $\tilde{s}$ is a successor of $s$ in $\M^{\mathbf{M}_{k}}$.
\end{proof}

\begin{lemma}
\label{lem:Dk}
Let $C^{\mathbf{M}_{k}}_1, \ldots, C^{\mathbf{M}_{k}}_{m_k}$ be all the SCCs of $\M^{\mathbf{M}_{k}}$
and for each $i \in \{1, \ldots, m_k\}$, let $D^{\mathbf{M}_{k}}_{p,1} = C^{\mathbf{M}_{k}}_i \times Q$.
Then, $\{D^{\mathbf{M}_{k}}_{p,1}, \ldots, D^{\mathbf{M}_{k}}_{p,m_{k}}\}$
is a partition of $S^{\mathbf{M}_{k}}_p$.
In addition, the following statements hold for all $i,j \in \{1, \ldots, m_k\}$.
\begin{itemize}
\item If $D^{\mathbf{M}_{k}}_{p,i}$ contains a state in an SCC $C^{\mathbf{M}_{k}}_p$ of $\M^{\mathbf{M}_{k}}_p$,
then it contains all the states in $C^{\mathbf{M}_{k}}_p$.
\item $Succ(D^{\mathbf{M}_{k}}_{p,i}) \cap D^{\mathbf{M}_{k}}_{p,j} \not= \emptyset$ only if 
$C^{\mathbf{M}_{k}}_j  \prec_{\M^{\mathbf{M}_{k}}} C^{\mathbf{M}_{k}}_i$.
\end{itemize}
\end{lemma}
\begin{proof}
Consider arbitrary $i,j \in \{1, \ldots, m_k\}$.
It follows directly from Lemma \ref{lem:Ck} that 
if $D^{\mathbf{M}_{k}}_{p,i}$ contains a state in an SCC $C^{\mathbf{M}_{k}}_p$ of $\M^{\mathbf{M}_{k}}_p$,
then it contains all the states in $C^{\mathbf{M}_{k}}_p$.
Next, consider the case where $Succ(D^{\mathbf{M}_{k}}_{p,i}) \cap D^{\mathbf{M}_{k}}_{p,j} \not= \emptyset$.
Then, from Lemma \ref{lem:Ck}, there exist SCCs
$C^{\mathbf{M}_{k}}_{p,i} \subseteq D^{\mathbf{M}_{k}}_{p,i}$ and 
$C^{\mathbf{M}_{k}}_{p,j} \subseteq D^{\mathbf{M}_{k}}_{p,j}$ of $\M^{\mathbf{M}_{k}}_p$
such that $Succ(C^{\mathbf{M}_{k}}_{p,i}) \cap C^{\mathbf{M}_{k}}_{p,j} \not= \emptyset$.
Thus, $C^{\mathbf{M}_{k}}_{p,j} \prec_{\M^{\mathbf{M}_{k}}_p} C^{\mathbf{M}_{k}}_{p,i}$.
Applying Lemma \ref{lem:Ck-order}, we get
$C^{\mathbf{M}_{k}}_{j} \prec_{\M^{\mathbf{M}_{k}}} C^{\mathbf{M}_{k}}_{i}$.
\end{proof}

Applying Lemma \ref{lem:Dk}, we generate a partition
$\{D^{\mathbf{M}_{k}}_{p,1}, \ldots, D^{\mathbf{M}_{k}}_{p,m_{k}}\}$ of $S^{\mathbf{M}_{k}}_p$
where for each $i \in \{1, \ldots, m_k\}$, $D^{\mathbf{M}_{k}}_{p,1} = C^{\mathbf{M}_{k}}_i \times Q$
and $C^{\mathbf{M}_{k}}_1, \ldots, C^{\mathbf{M}_{k}}_{m_k}$ are all the SCCs of $\M^{\mathbf{M}_{k}}$.
A partial order $\prec_{\M^{\mathbf{M}_{k}}_p}$ over this partition is defined such that
$D^{\mathbf{M}_{k}}_{p,j} \prec_{\M^{\mathbf{M}_{k}}_p} D^{\mathbf{M}_{k+1}}_{p,i}$
if $C^{\mathbf{M}_{k}}_j  \prec_{\M^{\mathbf{M}_{k}}} C^{\mathbf{M}_{k}}_i$.
Hence, an order $\mathbb{O}^{\mathbf{M}_{k}}_p$ among $D^{\mathbf{M}_{k}}_{p,1}, \ldots, D^{\mathbf{M}_{k}}_{p,m_{k}}$
can be simply derived from the order of $C^{\mathbf{M}_{k}}_{1}, \ldots, C^{\mathbf{M}_{k}}_{m_k}$,
which can be incrementally constructed based on Lemma \ref{lem:SCC-Mk+1:order}.
This order $\mathbb{O}^{\mathbf{M}_{k}}_p$ has the property that
the probability values of states in $D^{\mathbf{M}_{k}}_{p,j}$ that appears after $D^{\mathbf{M}_{k}}_{p,i}$
in $\mathbb{O}^{\mathbf{M}_{k}}_p$ cannot affect the probability values of states in $D^{\mathbf{M}_{k}}_{p,i}$.
Hence, we can follow the SCC-based value iteration
and process each $D^{\mathbf{M}_{k}}_{p,i}$ separately, according to the order in $\mathbb{O}^{\mathbf{M}_{k}}_p$
to compute the probability $x_s$ for all $s \in D^{\mathbf{M}_{k}}_{p,i}$.
Finally, we generate a memoryless control policy $\C^{\mathbf{M}_{k}}$
from the probability vector $(x_s)_{s \in S^{\mathbf{M}_{k}}_p}$ as described at the end of Section \ref{sec:syn}.

\section{Experimental Results}
\label{sec:example}

Consider, once again, the autonomous vehicle problem described in Example \ref{ex:ped-models}
and Example \ref{ex:ped-spec}.
Suppose the road is discretized into 5 cells $c_0, \ldots, c_4$ where $c_2$ is the pedestrian crossing area
as shown in Figure \ref{fig:example}.
The vehicle starts in cell $c_0$ and has to reach cell $c_4$.
There are 5 pedestrians, modeled by MCs $\M_1, \ldots, \M_5$, initially at cell $c_1$.
The models of the vehicle and the pedestrians are shown in Figure \ref{fig:models}.
A DFA $\A_\varphi$ that accepts all and only words in $pref(\varphi)$ where
$\varphi = \left(\neg \bigvee_{i \geq 1, j \geq 0} (c^0_j \aand c^i_j)\right) \until c^0_4$
is shown in Figure \ref{fig:Aphi}.

\begin{figure}
\centering\includegraphics[width=0.3\textwidth]{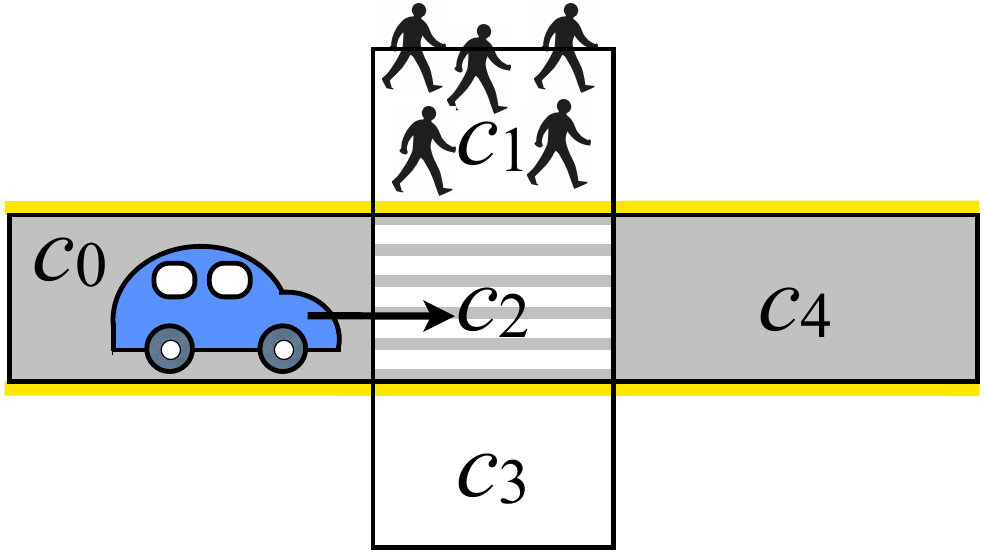}
\caption{The road and its partition used in the autonomous vehicle example.}
\label{fig:example}
\end{figure}

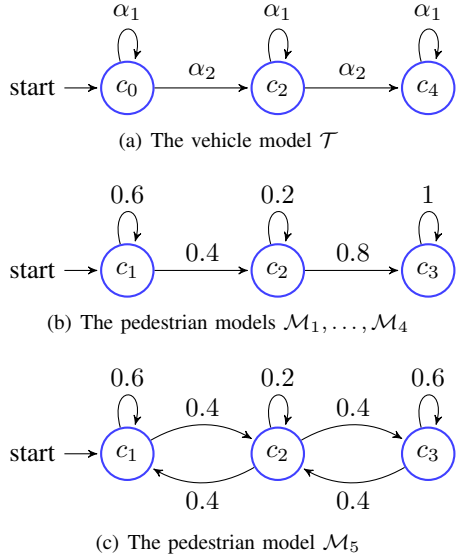
\begin{figure}
\centering 
\subfigure[The vehicle model $\T$]{
	\begin{tikzpicture}[->,>=stealth',shorten >=1pt,auto,node distance=2cm]
	  \tikzstyle{every state}=[circle,thick,draw=blue!75,minimum size=6mm]

	   \node [initial,state] (c0) at (1,0) {$c_0$};
	   \node [state] (c2) at (3,0){$c_2$};
	   \node [state] (c4) at (5,0) {$c_4$};
	   
	   \path (c0) edge [loop above] node {$\alpha_1$} (c0)
	                     edge node [] {$\alpha_2$} (c2)
	             (c2) edge [loop above] node {$\alpha_1$} (c2)
	                     edge node [] {$\alpha_2$} (c4)
	             (c4) edge [loop above] node {$\alpha_1$} (c4);
	\end{tikzpicture}
}

\subfigure[The pedestrian models $\M_1, \ldots, \M_4$]{
	\begin{tikzpicture}[->,>=stealth',shorten >=1pt,auto,node distance=2cm]
	  \tikzstyle{every state}=[circle,thick,draw=blue!75,minimum size=6mm]

	   \node [initial,state] (c1) at (1,0) {$c_1$};
	   \node [state] (c2) at (3,0){$c_2$};
	   \node [state] (c3) at (5,0) {$c_3$};
	   
	   \path (c1) edge [loop above] node {$0.6$} (c1)
	                     edge node [] {$0.4$} (c2)
	             (c2) edge [loop above] node {$0.2$} (c2)
	                     edge node [] {$0.8$} (c3)
	             (c3) edge [loop above] node {$1$} (c3);
	\end{tikzpicture}
}

\subfigure[The pedestrian model $\M_5$]{
	\begin{tikzpicture}[->,>=stealth',shorten >=1pt,auto,node distance=2cm]
	  \tikzstyle{every state}=[circle,thick,draw=blue!75,minimum size=6mm]

	   \node [initial,state] (c1) at (1,0) {$c_1$};
	   \node [state] (c2) at (3,0){$c_2$};
	   \node [state] (c3) at (5,0) {$c_3$};
	   
	   \path (c1) edge [loop above] node {$0.6$} (c1)
	                     edge [bend left] node {$0.4$} (c2)
	             (c2) edge [loop above] node {$0.2$} (c2)
	                     edge [bend left]  node{$0.4$} (c3)
	                     edge [bend left] node {$0.4$} (c1)
	             (c3) edge [loop above] node {$0.6$} (c3)
	                     edge [bend left] node {$0.4$} (c2);
	\end{tikzpicture}
}

   \caption{The models of vehicle and pedestrians.}
   \label{fig:models}
\end{figure}

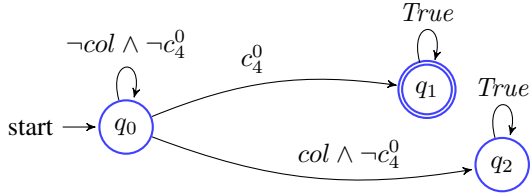
\begin{figure}
\centering 
\begin{tikzpicture}[->,>=stealth',shorten >=1pt,auto,node distance=2cm, bend angle=15]
  \tikzstyle{every state}=[circle,thick,draw=blue!75,minimum size=6mm]

   \node [initial,state] (q0) at (1,0) {$q_0$};
   \node [state, double] (q1) at (5,0.5){$q_1$};
   \node [state] (q2) at (6,-0.5) {$q_2$};
   
   \path (q0) edge [loop above] node {$\neg col \aand \neg c^0_4$} (q0)
                     edge [bend left] node {$c^0_4$} (q1)
                     edge [bend right] node {$\hspace{-3mm}col \aand \neg c^0_4$} (q2)
             (q1) edge [loop above] node {$\true$} (q1)
             (q2) edge [loop above] node {$\true$} (q2);
\end{tikzpicture}
   \caption{A DFA $\A_\varphi$ that recognizes the prefixes of
   $\varphi = \neg col \until c^0_4$ where $col$ is defined as $col = \bigvee_{i \geq 1, j \geq 0} (c^0_j \aand c^i_j)$.
   $q_1$ is the accepting state.}
   \label{fig:Aphi}
\end{figure}

First, we apply the LP-based, value iteration and SCC-based value iteration techniques described in Section \ref{sec:syn}
to synthesize a control policy that maximizes the probability that the complete system $\M = \T || \M_1 || \M_2 || \ldots || \M_5$ satisfies $\varphi$.
The time required for each step of computation is summarized in Table \ref{tab:time-full}.
All the approaches yield the probability of 0.8 that $\M$ satisfies $\varphi$ under the synthesized control policy. 
The comparison of the total computation time required for these different approaches is shown in Figure \ref{fig:result_5ped}.
As discussed in Section \ref{ssec:policy}, although the SCC-based value iteration itself takes
significantly less computation time than the LP-based technique or value iteration,
the time spent in identifying SCCs and their order renders the total computation time of the SCC-based value iteration
more than the other two approaches.

\begin{table}[h]
\centering
\begin{tabular}{ | c | c | c | c | c | c |}
  \hline                
  Technique & 
  \hspace{-3mm}$\M_p$\hspace{-3mm} & 
  \hspace{-3mm}$\begin{array}{c}\hbox{SCCs}\\ \hbox{\& order}\\ \hbox{of } \M_p\end{array}$\hspace{-3mm} & 
  \hspace{-2mm}$\begin{array}{c}\hbox{Prob}\\ \hbox{vector}\end{array}$\hspace{-2mm} & 
  \hspace{-2mm}$\begin{array}{c}\hbox{Control}\\ \hbox{policy}\end{array}$\hspace{-2mm} & 
  \hspace{-1mm}Total\hspace{-1mm} \\
  \hline
  \hline  
  LP & 
  156.3 & - & 8.8 & 6.8 & 171.9\\
  \hline
  \hspace{-3mm}$\begin{array}{c}\hbox{Value}\\ \hbox{iteration}\end{array}$\hspace{-3mm} &
  156.3 & - & 31.3 & 6.8 & 194.4\\
  \hline
  \hspace{-3mm}$\begin{array}{c}\hbox{SCC-based}\\ \hbox{value iteration}\end{array}$\hspace{-3mm} &
  156.3 & 71.1 & 1.9 & 6.8 & 236.1\\
  \hline  
\end{tabular}
\caption{Time required (in seconds) for computing various objects using different techniques when the full models of
all the environment agents are considered.}
\label{tab:time-full}
\end{table}

Next, we apply the incremental technique where we progressively compute a sequence of control policies as more agents
are added to the synthesis procedure in each iteration 
as described in Section \ref{sec:incremental}.
We let $\mathbf{M}_{0} = \emptyset$, $\mathbf{M}_{1} = \{\M_1\}$, $\mathbf{M}_{2} = \{\M_1, \M_2\}, \ldots$, 
$\mathbf{M}_{6} = \{\M_1, \ldots, \M_5\}$,
i.e., we successively add each pedestrian $\M_1, \M_2, \ldots, \M_5$, respectively, in each iteration.
We consider 2 cases:
(1) no incremental construction of various objects is employed (i.e., when $\M^{\mathbf{M}_{k+1}}$ and $\M^{\mathbf{M}_{k+1}}_p$,
$k \geq 0$ are computed from scratch in every iteration), and
(2) incremental construction of various objects as described in Section \ref{ssec:inc_pMDP}--\ref{ssec:inc_prob} is applied.
For the first case, we apply the LP-based technique to compute the probability vector as
it has been shown to be the fastest technique when applied to this problem, 
taking into account the required pre-computation, which needs to be done in every iteration.
For both cases, 6 control policies $\C^{\mathbf{M}_0}, \ldots, \C^{\mathbf{M}_5}$ are generated 
for $\M^{\mathbf{M}_{0}}, \ldots, \M^{\mathbf{M}_{5}}$, respectively.
For each policy $\C^{\mathbf{M}_k}$, we compute the probability $\prob^{\C^{\mathbf{M}_k}}_{\M}(\varphi)$
that the complete system $\M$ satisfies $\varphi$ under policy $\C^{\mathbf{M}_k}$.
(Note that $\C^{\mathbf{M}_k}$, when applied to $\M$, is only a function of states of $\M_i \in \mathbf{M}_k$ since
it assumes that the other agents $\M_j \not\in \mathbf{M}_k$ are stationary.)
These probabilities are given by
$\prob^{\C^{\mathbf{M}_0}}_{\M}(\varphi) = 0.08$,
$\prob^{\C^{\mathbf{M}_1}}_{\M}(\varphi) = 0.46$,
$\prob^{\C^{\mathbf{M}_2}}_{\M}(\varphi) = 0.57$,
$\prob^{\C^{\mathbf{M}_3}}_{\M}(\varphi) = 0.63$,
$\prob^{\C^{\mathbf{M}_4}}_{\M}(\varphi) = 0.67$ and
$\prob^{\C^{\mathbf{M}_5}}_{\M}(\varphi) = 0.8$.

The comparison of the cases where the incremental construction of various objects is not and is employed is shown in Figure \ref{fig:result_5ped}.
A jump in the probability occurs each time a new control policy is computed.
The time spent during each step of computation is summarized in Table \ref{tab:time-noninc} and Table \ref{tab:time-inc}
for the first and the second case, respectively.
Notice that the time required for identifying the SCCs and their order when
the incremental approach is applied is significantly less than
when the full model of all the pedestrians is considered in one shot 
since $\M^{\mathbf{M}_{0}}$, $\M_1, \ldots, \M_5$, each of which contains 3 states, 
are much smaller than $\M_p$, which contains 2187 states.

\begin{figure}
   \centering 
   \includegraphics[width=0.5\textwidth]{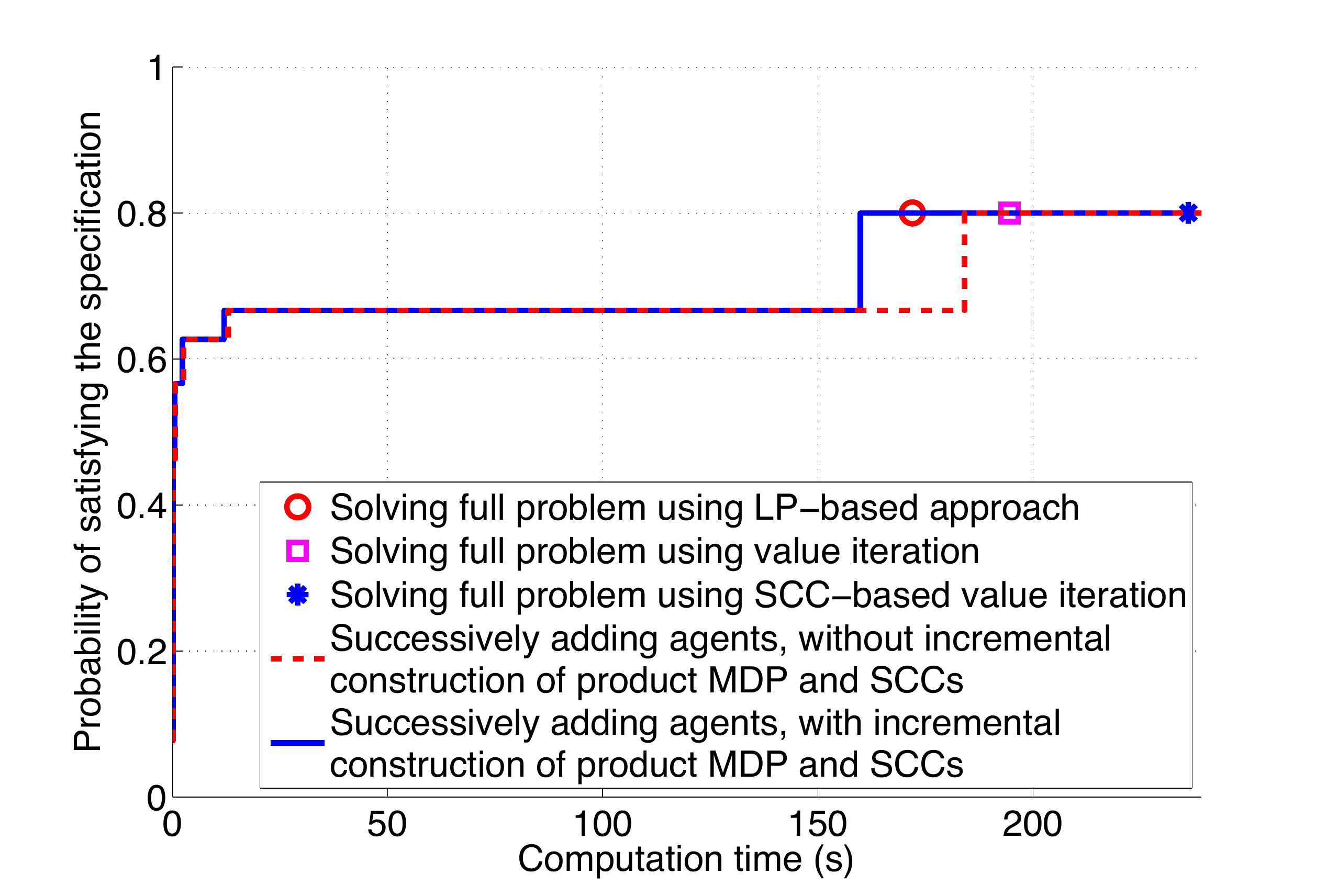}
   \caption{Comparison of the computation time and the probability for the system to satisfy the specification computed
   using different techniques.}
   \label{fig:result_5ped} 
\end{figure}

From Figure \ref{fig:result_5ped}, our incremental approach is able to obtain
an optimal control policy faster than any other techniques.
This is mainly due to the efficiency of our incremental construction of SCCs and their order.
In addition, we are able to obtain a reasonable solution,
with 0.67 probability of satisfying $\varphi$,
within 12 seconds while the maximum probability of satisfying $\varphi$ is 0.8, which requires 160 seconds of computation
(or 171.9 seconds without employing the incremental approach).

\begin{table}[h]
\centering
\begin{tabular}{ | c | c | c | c | c | c |}
  \hline                
  \hspace{-1mm}Iteration\hspace{-1mm} & 
  \hspace{-2mm}$\M^{\mathbf{M}_{k}}$\hspace{-2mm} & 
  \hspace{-2mm}$\M^{\mathbf{M}_{k}}_p$\hspace{-2mm} & 
  \hspace{-2mm}$\begin{array}{c}\hbox{Prob}\\ \hbox{vector}\end{array}$\hspace{-2mm} & 
  \hspace{-2mm}$\begin{array}{c}\hbox{Control}\\ \hbox{policy}\end{array}$\hspace{-2mm} & 
  \hspace{-1mm}Total\hspace{-1mm} \\
  \hline
  \hline  
  0 & 0.0064 & 0.0185 & 0.0464 & 0.0084 & 0.08\\
  \hline
  1 & 0.0123 & 0.0762 & 0.0203 & 0.0104 & 0.12\\
  \hline
  2 & 0.0154 & 0.3383 & 0.0231 & 0.0296 & 0.41\\
  \hline
  3 & 0.0357 & 1.7055 & 0.0542 & 0.1503 & 1.95\\
  \hline
  4 & 0.1393 & 9.1950 & 0.2155 & 0.7975 & 10.35\\
  \hline
  5 & 3.1836 & 152.86 & 8.2302 & 6.8938 & 171.17\\
  \hline  
\end{tabular}
\caption{Time required (in seconds) for computing various objects in each iteration when incremental construction is not applied.}
\label{tab:time-noninc}
\end{table}

\begin{table}[h]
\centering
\begin{tabular}{ | c | c | c | c | c | c | c |}
  \hline                
  \hspace{-3mm}$\begin{array}{c}\hbox{Iter-}\\ \hbox{ation}\end{array}$\hspace{-3mm} & 
  \hspace{-2mm}$\M^{\mathbf{M}_{0}}$\hspace{-2mm} & 
  \hspace{-3mm}$\begin{array}{c}\hbox{SCCs \& order}\\ \hbox{of } \M^{\mathbf{M}_{0}},\\ \M_1, \ldots, \M_5\end{array}$\hspace{-3mm} & 
  \hspace{-3mm}$\begin{array}{c}\M^{\mathbf{M}_{k}}_p,\\ \hbox{partition}\\ \hbox{\& order}\end{array}$\hspace{-3mm} & 
  \hspace{-2mm}$\begin{array}{c}\hbox{Prob}\\ \hbox{vector}\end{array}$\hspace{-2mm} & 
  \hspace{-2mm}$\begin{array}{c}\hbox{Control}\\ \hbox{policy}\end{array}$\hspace{-2mm} & 
  \hspace{-1mm}Total\hspace{-1mm} \\
  \hline
  \hline  
  0 & \hspace{-1mm}0.0055\hspace{-1mm} & 0.0043 & 0.0203     & 0.0112 & 0.0036   & 0.04\\
  \hline
  1 & -             & -            & 0.0726     & 0.0102 & 0.0087   & 0.09\\
  \hline
  2 & -             & -            & 0.3239     & 0.0193 & 0.0282   & 0.37\\
  \hline
  3 & -             & -            & 1.6036     & 0.0567 & 0.1424   & 1.80\\
  \hline
  4 & -             & -            & 8.6955     & 0.1876 &  0.7755  & 9.66 \\
  \hline
  5 & -             & -            & 139.27     & 1.6122 &  7.0125  & 147.89\\
  \hline  
\end{tabular}
\caption{Time required (in seconds) for computing various objects in each iteration when incremental construction is applied.}
\label{tab:time-inc}
\end{table}

\section{Conclusions and Future Work}
\label{sec:conclusions}
An anytime algorithm for synthesizing a control policy for a robot interacting with multiple environment agents
with the objective of maximizing the probability for the robot to satisfy a given temporal logic specification was proposed.
Each environment agent is modeled by a Markov chain whereas the robot is modeled by a finite transition system 
(in the deterministic case) or Markov decision process (in the stochastic case).
The proposed algorithm progressively computes a sequence of control policies, taking into account only a small subset of the environment agents initially 
and successively adding more agents to the synthesis procedure in each iteration until we hit the constraints on computational resources.
Incremental construction of various objects needed to be computed during the synthesis procedure was proposed.
Experimental results showed that not only we obtain a reasonable solution much faster than
existing approaches, but we are also able to obtain an optimal solution faster than existing approaches.

Future work includes extending the algorithm to handle full LTL specifications.
This direction appears to be promising because the remaining step is only to
incrementally construct \emph{accepting maximal end components} of an MDP.
We are also examining an effective approach to determine an agent to be added in each iteration.
As mentioned in Section \ref{ssec:overview}, such an agent may be picked based on the result
from probabilistic verification but this comes at the extra cost of adding the verification phase.

\bibliographystyle{ieeetr}
\bibliography{ref}

\end{document}